\documentclass[letterpaper]{article} 
\usepackage[draft]{aaai25}  
\usepackage{times}  
\usepackage{helvet}  
\usepackage{courier}  
\usepackage[hyphens]{url}  
\usepackage{graphicx} 
\usepackage{natbib}  
\usepackage{caption} 
\usepackage{algorithm}
\usepackage{algorithmic}

\usepackage{newfloat}
\usepackage{listings}
\DeclareCaptionStyle{ruled}{labelfont=normalfont,labelsep=colon,strut=off} 
\floatstyle{ruled}
\newfloat{listing}{tb}{lst}{}
\floatname{listing}{Listing}
\usepackage{amsmath,amssymb,amsfonts}
\usepackage{algorithmic}
\usepackage{graphicx}
\usepackage{textcomp}
\usepackage{xcolor}
\def\BibTeX{{\rm B\kern-.05em{\sc i\kern-.025em b}\kern-.08em
    T\kern-.1667em\lower.7ex\hbox{E}\kern-.125emX}}

\usepackage[utf8]{inputenc} 
\usepackage[T1]{fontenc}    
\usepackage{url}            
\usepackage{booktabs}       
\usepackage{makecell}       
\usepackage{amsfonts}       
\usepackage{nicefrac}       
\usepackage{microtype}      
\usepackage{xcolor}         

\usepackage{mathtools}
\usepackage{amsthm}
\usepackage{thmtools}
\usepackage{thm-restate}

\usepackage[capitalize,noabbrev,nameinlink]{cleveref}

\usepackage[]{caption}
\usepackage{subcaption}
\usepackage{float}
\usepackage{graphicx,xcolor,graphics,url}
\usepackage{verbatim,ifthen,alltt,array}
\usepackage{latexsym,dsfont}
\usepackage{booktabs,multirow}
\usepackage{enumitem}
\usepackage{xparse}

\usepackage{xspace}

\usepackage{xfrac}
\usepackage[framemethod=TikZ]{mdframed}
\usepackage{nicefrac}       

\usepackage{stmaryrd}
\usepackage{textcomp}
\usepackage{siunitx}
\usepackage{xpatch}

\usepackage{pgf}
\usepackage{tikz}
\usepackage{forest}

\usetikzlibrary{arrows,decorations.pathmorphing,decorations.pathreplacing,decorations.footprints,fadings,calc,trees,mindmap,shadows,decorations.text,patterns,positioning,shapes,matrix,fit}
\usetikzlibrary{arrows,shadows,backgrounds}
\usetikzlibrary{arrows.meta}
\usetikzlibrary{positioning,fit}
\usetikzlibrary{automata}
\usetikzlibrary{matrix}
\usetikzlibrary{shapes.symbols,shapes.misc,shapes.arrows}
\usetikzlibrary{matrix,chains,scopes,decorations.pathmorphing}
\usetikzlibrary{bayesnet}
\usetikzlibrary{tikzmark}

\definecolor{darkred}{rgb}{0.7,0.1,0.1}
\definecolor{medred}{rgb}{0.5,0.1,0.1}
\definecolor{midred}{rgb}{0.7,0.2,0.2}
\definecolor{vdarkred}{rgb}{0.4,0.1,0.1}
\definecolor{darkslategray}{rgb}{0.18, 0.31, 0.31} 
\definecolor{platinum}{rgb}{0.9, 0.89, 0.89} 
\definecolor{gray}{rgb}{.4,.4,.4}
\definecolor{midgrey}{rgb}{0.5,0.5,0.5}
\definecolor{middarkgrey}{rgb}{0.35,0.35,0.35}
\definecolor{darkgrey}{rgb}{0.3,0.3,0.3}
\definecolor{darkred}{rgb}{0.7,0.1,0.1}
\definecolor{midblue}{rgb}{0.2,0.2,0.7}
\definecolor{darkblue}{rgb}{0.1,0.1,0.5}
\definecolor{darkgreen}{rgb}{0.1,0.5,0.1}
\definecolor{defseagreen}{cmyk}{0.69,0,0.50,0}
\definecolor{purple3}{RGB}{125,38,205}          

\definecolor{tyellow1}{HTML}{FCE94F}
\definecolor{tyellow2}{HTML}{EDD400}
\definecolor{tyellow3}{HTML}{C4A000}
\definecolor{torange1}{HTML}{FCAF3E}
\definecolor{torange2}{HTML}{F57900}
\definecolor{torange3}{HTML}{C35C00}
\definecolor{tbrown1}{HTML}{E9B96E}
\definecolor{tbrown2}{HTML}{C17D11}
\definecolor{tbrown3}{HTML}{8F5902}
\definecolor{tgreen1}{HTML}{8AE234}
\definecolor{tgreen2}{HTML}{73D216}
\definecolor{tgreen3}{HTML}{4E9A06}
\definecolor{tblue1}{HTML}{729FCF}
\definecolor{tblue2}{HTML}{3465A4}
\definecolor{tblue3}{HTML}{204A87}
\definecolor{tpurple1}{HTML}{AD7FA8}
\definecolor{tpurple2}{HTML}{75507B}
\definecolor{tpurple3}{HTML}{5C3566}
\definecolor{tred1}{HTML}{EF2929}
\definecolor{tred2}{HTML}{CC0000}
\definecolor{tred3}{HTML}{A40000}
\definecolor{tlgray1}{HTML}{EEEEEC}
\definecolor{tlgray2}{HTML}{D3D7CF}
\definecolor{tlgray3}{HTML}{BABDB6}
\definecolor{tdgray1}{HTML}{888A85}
\definecolor{tdgray2}{HTML}{555753}
\definecolor{tdgray3}{HTML}{2E3436}

\newtheoremstyle{nthmstyle}
{3pt}
{3pt}
{}
{}
{\bfseries}
{.}
{.5em}
{}

\theoremstyle{nthmstyle}

\newtheorem{proposition}{Proposition}
\newtheorem{example}{Example}

\crefname{enumi}{}{}
\crefname{rstprop}{Proposition}{Propositions}

\newcommand{\fml}[1]{{\mathcal{#1}}}

\newcommand{\tn}[1]{\textnormal{#1}}

\newcommand{\msf}[1]{\ensuremath\mathsf{#1}}
\newcommand{\mbf}[1]{\ensuremath\mathbf{#1}}
\newcommand{\mbb}[1]{\ensuremath\mathbb{#1}}
\newcommand{\oper}[1]{\ensuremath\mathsf{#1}}

\newcommand{\waxp}{\ensuremath\mathsf{WAXp}}
\newcommand{\wcxp}{\ensuremath\mathsf{WCXp}}
\newcommand{\axp}{\ensuremath\mathsf{AXp}}
\newcommand{\cxp}{\ensuremath\mathsf{CXp}}
\newcommand{\aex}{\ensuremath\mathsf{AEx}}

\newcommand{\relevant}{\oper{Relevant}}
\newcommand{\irrelevant}{\oper{Irrelevant}}

\newcommand{\nfrac}{\nicefrac}

\newcommand{\exv}{\ensuremath\mathbf{E}}

\newcommand{\cf}{\ensuremath\upsilon} 
\newcommand{\cfn}[1]{\ensuremath\upsilon_{#1}} 

\newcommand{\svn}[1]{\msf{Sc}_{#1}}

\newcommand{\mdist}{\ensuremath\mathfrak{d}}

\newcommand{\similar}{\ensuremath\sigma}

\DeclareMathOperator*{\sv}{\msf{Sc}}

\definecolor{gray}{rgb}{.4,.4,.4}
\definecolor{midgrey}{rgb}{0.5,0.5,0.5}
\definecolor{middarkgrey}{rgb}{0.35,0.35,0.35}
\definecolor{darkgrey}{rgb}{0.3,0.3,0.3}
\definecolor{darkred}{rgb}{0.7,0.1,0.1}
\definecolor{midblue}{rgb}{0.2,0.2,0.7}
\definecolor{darkblue}{rgb}{0.1,0.1,0.5}
\definecolor{defseagreen}{cmyk}{0.69,0,0.50,0}

\newcommand{\jnoteF}[1]{}

\newcounter{Comment}[Comment]
\declaretheoremstyle[
  headfont=\bfseries,
  bodyfont=\itshape,
  numberwithin=section,
]{StdThmStyle}

\tikzset{
  0 my edge/.style={densely dashed, my edge},
  my edge/.style={-{Stealth[]}},
}

\setlist{nosep}

\makeatletter
\newcommand\nparagraph{%
  \@startsection{paragraph}
    {4}
    {\z@}
    {0.225ex \@plus0.225ex \@minus.125ex}
    {-1em}
    {\normalfont\normalsize\bfseries}%
}
\makeatother

\begin{document}


\title{SHAP scores fail pervasively even when Lipschitz succeeds}

\author {
  Olivier L\'{e}toff\'{e}\textsuperscript{\rm 1},
  Xuanxiang Huang\textsuperscript{\rm 2},
  Joao Marques-Silva\textsuperscript{\rm 3}
}
\affiliations {
    \textsuperscript{\rm 1}Univ.~Toulouse, France\\
    \textsuperscript{\rm 2}CNRS@CREATE, Singapore\\
    \textsuperscript{\rm 3}ICREA, Univ.~Lleida, Spain\\
    olivier.letoffe@orange.fr,
    xuanxiang.huang.cs@gmail.com,
    jpms@icrea.cat
}

\maketitle

\begin{abstract}
  The ubiquitous use of Shapley values in eXplainable AI (XAI) has
  been triggered by the tool SHAP, and as a result are commonly
  referred to as SHAP scores.
  %
  %
  Recent work devised examples of machine learning (ML) classifiers
  for which the computed SHAP scores are thoroughly unsatisfactory, by
  allowing human decision-makers to be misled.
  Nevertheless, such examples could be perceived as somewhat
  artificial, since the selected classes must be interpreted as
  numeric.
  Furthermore, it was unclear how general were the issues identified
  with SHAP scores.
  This paper answers these criticisms.
  First, the paper shows that for Boolean classifiers there are
  arbitrarily many examples for which the SHAP scores must be deemed
  unsatisfactory.
  Second, the paper shows that the issues with SHAP scores are also
  observed in the case of regression models.
  In addition, the paper studies the class of regression models that
  respect Lipschitz continuity, a measure of a function's rate of
  change that finds important recent uses in ML, including model
  robustness.
  Concretely, the paper shows that the issues with SHAP scores occur
  even for regression models that respect Lipschitz continuity.
  %
  %
  Finally, the paper shows that the same issues are guaranteed to
  exist for arbitrarily differentiable regression models.
\end{abstract}

\section{Introduction} \label{sec:intro}

SHAP scores~\cite{lundberg-nips17} denote the use of Shapley values in
eXplainable AI (XAI).%
\footnote{Proposed in the 1950s~\cite{shapley-ctg53}, Shapley values
find important uses in different domains, that include game theory
among others~\cite{roth-bk88}.}
Popularized by the tool SHAP~\cite{lundberg-nips17}, SHAP scores find
an ever-increasing range of practical uses, that have continued to
expand at a fast pace in recent years~\cite{molnar-bk23,mishra-bk23}.
Many proposed applications of SHAP scores impact human beings, 
in different ways, and so are categorized as being
high-risk~\cite{euaict24}.
Given the advances in machine learning, and the proposals for
regulating its use~\cite{alegai19,wheo23,euaict24}, one should expect
the uses of SHAP scores to continue to increase.

Despite their immense success, SHAP scores face a number of well-known
limitations. First, there are practical limitations, as exemplified by
a number of recent
works~\cite{blockbaum-aistats20,friedler-icml20,najmi-icml20,friedler-nips21}. Second,
as reported in recent work~\cite{hms-corr23a}, there are theoretical
limitations. Concretely, there exist machine learning (ML) classifiers
for which the computed SHAP scores are thoroughly unsatisfactory given
self-evident properties of those classifiers.
The theoretical limitations of SHAP scores should be perceived as
rather 
problematic, since they challenge the foundations of SHAP scores, and
so also serve to question the experimental results obtained in massive
numbers of practical uses of SHAP scores.

Nevertheless, the example classifiers studied in earlier
work~\cite{hms-corr23a} can be criticized in different ways.
First, it is unclear whether those classifers are easy to devise,
i.e.\ it might be the case that the identified examples were selected
from among a small number of possible cases.
Second, the theoretical limitations of SHAP scores were proved for the
concrete case of ML classifiers, where classes are viewed as numeric
values~\cite{hms-corr23a}. Still in many important uses of ML, the
computed output is most often a real value (e.g.\ some
probability). As a result, one might hope that the theoretical
limitations of SHAP scores would not apply in those settings.

This paper proves that such hypothetical criticisms are without merit.
%
First, the paper proves that there are arbitrarily many boolean
classifiers which will exhibit one or more of the issues identified
with SHAP scores.
Second, the paper demonstrates the existence of regression models for 
which the computed SHAP scores are also entirely unsatisfactory.
Second, and given the importance of Lipschitz continuity in
adversarial robustness~\cite{kornblith-icml21,allgower-csl22}, the
paper also shows the existence of Lipschitz-continuous regression
models for which the computed SHAP scores are guaranteed to (again) be
unsatisfactory. Furthermore, some of the proposed regression models
are parameterized, meaning that the number of examples of regression
models with unsatisfactory
SHAP scores is unbounded, and the degree of dissatisfaction can be
made arbitrarily large. Third, and finally, the paper shows that 
additional regression models
are guaranteed to exist for arbitrarily differentiable regression
functions.
To prove the above results, the paper exploits a generalization of the
standard definitions in formal
explainability~\cite{kutyniok-jair21,msi-aaai22,darwiche-lics23}, to 
account for regression models, especially when those regression models
are defined on real-value features. 


\section{Preliminaries} \label{sec:prelim}

\subsubsection{Classification \& regression problems.}
%
Let $\fml{F}=\{1,\ldots,m\}$ denote a set of features.
Each feature $i\in\fml{F}$ takes values from a domain $\mbb{D}_i$.
Domains can be categorical or ordinal. If ordinal, domains can be
discrete or real-valued.
Feature space is defined by
$\mbb{F}=\mbb{D}_1\times\mbb{D}_2\times\ldots\times\mbb{D}_m$. 
The notation $\mbf{x}=(x_1,\ldots,x_m)$ denotes an arbitrary point in 
feature space, where each $x_i$ is a variable taking values from
$\mbb{D}_i$. Moreover, the notation $\mbf{v}=(v_1,\ldots,v_m)$
represents a specific point in feature space, where each $v_i$ is a
constant representing one concrete value from $\mbb{D}_i$.
A classifier maps each point in feature space to a class taken from
$\fml{K}=\{c_1,c_2,\ldots,c_K\}$. Classes can also be categorical or
ordinal.
A boolean classifier denotes the situation where
$\mbb{D}_i=\mbb{B}=\{0,1\}$.
In the case of regression, each point in feature space is mapped to an
ordinal value taken from a set $\mbb{K}$, e.g.\ $\mbb{K}$ could denote
$\mbb{Z}$ or $\mbb{R}$.
Therefore, a classifier $\fml{M}_{C}$ is characterized by a
non-constant \emph{classification function} $\kappa$ that maps feature
space $\mbb{F}$ into the set of classes $\fml{K}$,
i.e.\ $\kappa:\mbb{F}\to\fml{K}$.
A regression model $\fml{M}_R$ is characterized by a non-constant
\emph{regression function} $\rho$ that maps feature space $\mbb{F}$
into the set elements from $\mbb{K}$, i.e.\ $\rho:\mbb{F}\to\mbb{K}$.
A classifier model $\fml{M}_{C}$ is represented by a tuple
$(\fml{F},\mbb{F},\fml{K},\kappa)$, whereas a regression model
$\fml{M}_{R}$ is represented by a tuple
$(\fml{F},\mbb{F},\mbb{K},\rho)$.
%
%
When viable, we will represent an ML model $\fml{M}$ by a tuple
$(\fml{F},\mbb{F},\mbb{T},\tau)$, with $\tau:\mbb{F}\to\mbb{T}$,
without specifying whether $\fml{M}$ denotes a classification
or a regression model.
A \emph{sample} (or instance) denotes a pair $(\mbf{v},q)$, where
$\mbf{v}\in\mbb{F}$ and either $q\in\fml{K}$, with
$q=\kappa(\mbf{v})$, or $q\in\mbb{K}$, with $q=\rho(\mbf{v})$.
%

\subsubsection{Running example(s).}
%
\cref{ex:tr:rt} shows a function in tabular representation (TR) that will be used as one of the
running examples, we denote this model as $\fml{M}_{1}$.
Clearly, $\fml{F}=\{1,2\}$, $\mbb{D}_1=\mbb{D}_2=\{0,1\}$,
$\mbb{F}=\mbb{D}_1\times\mbb{D}_2$.
Depending on the value of $\alpha\in\mbb{R}$, the model $\fml{M}_{1}$ will
be viewed as a classification or as a regression model~\cite{james-bk17}.
As a result, we will refer to the function computed by $\fml{M}_{1}$
as $\tau_1:\mbb{F}\to\mbb{T}$.
If the values in the last column $\tau_1(\mbf{x})$ are integers,
then we say that we have a classification model, with the set of
classes $\fml{K}$ given by the values in the last column,
and so $\tau_1$ corresponds to a classification function $\kappa_1:\mbb{F}\to\fml{K}$,
with $\fml{K}=\{1-6\alpha,1,1+2\alpha\}$.
Otherwise, if some values in the last column $\tau_1(\mbf{x})$
are real numbers, then we say that we have a regression
model, with the (finite) set of values in the codomain given by the
values in the last column, and so $\tau_1$ corresponds to a regression
function $\rho_1:\mbb{F}\to\mbb{K}$, with $\mbb{K}=\{1-6\alpha,1,1+2\alpha\}$.
Besides the function shown in~\cref{ex:tr:rt}, additional running examples
will be introduced later in the paper.

%
%
%

\begin{figure}[t]
\centering
\renewcommand{\tabcolsep}{0.45em}
\begin{tabular}{cccc} \toprule
  row & $x_1$ & $x_2$ & $\tau_1(\mbf{x})$
  \\ \toprule
  $1$ & $0$ & $0$ & $1-6\alpha$ \\
  $2$ & $0$ & $1$ & $1+2\alpha$ \\
  $3$ & $1$ & $0$ & $1$           \\
  \tikzmarknode{a}{$4$} & $1$ & $1$ & \tikzmarknode{b}{$1$}
  \\ \bottomrule
  \begin{tikzpicture}[overlay,remember picture]
    \node[draw=midblue, thin, xshift=-0.35pt, yshift=-0.35pt, inner
      sep=2.0pt, fit=(a) (b)] {};
  \end{tikzpicture}
\end{tabular}

\caption{Tabular representation (TR) of a simple function.
The target sample is $((1,1),1)$.}
\label{ex:tr:rt}
\end{figure}

\subsubsection{Lipschitz continuity.}
%
Let $(\mbb{F},\mdist_F)$ and $(\mbb{R},\mdist_R)$ denote metric
spaces.%
\footnote{%
$\mdist_F:\mbb{F}\times\mbb{F}\to\mbb{R}$
and 
$\mdist_R:\mbb{R}\times\mbb{R}\to\mbb{R}$
denote distance functions between two points, which we refer to as
$\mdist$. A distance function $\mdist$ respects the well-known axioms:  
(i) $\mdist(\mbf{x},\mbf{x})=0$;
(ii) if $\mbf{x}\not=\mbf{y}$, then $\mdist(\mbf{x},\mbf{y})>0$;
(iii) $\mdist(\mbf{x},\mbf{y})=\mdist(\mbf{y},\mbf{x})$; and
(iv)
$\mdist(\mbf{x},\mbf{z})\le\mdist(\mbf{x},\mbf{y})+\mdist(\mbf{y},\mbf{z})$.
Examples of distance functions include Hamming, Manhattan, Euclidean
and other distance defined by norm $l_p$, $p\ge1$,
where $\lVert\mbf{x}\rVert_{p}:=\left(\sum\nolimits_{i=1}^{m}|x_i|^{p}\right)^{\sfrac{1}{p}}$.
%
}
A regression function $\rho:\mbb{F}\to\mbb{R}$ is
Lipschitz-continuous~\cite{osearcoid-bk06} if there exists a constant
$C\ge0$ such that,
\begin{equation} \label{def:lipc}
  \forall(\mbf{x}_1,\mbf{x}_2\in\mbb{F}).
  \mdist_R(\rho(\mbf{x}_1),\rho(\mbf{x}_2))\le{C}\mdist_F(\mbf{x}_1,\mbf{x}_2),
\end{equation}
where $C$ is referred to as the Lipschitz constant.
It is well-known that any Lipschitz-continuous function is also
continuous.

The relationship between Lipschitz continuity and adversarial
robustness has been acknowledged for more than a
decade~\cite{szegedy-iclr14}, i.e.\ since the brittleness of ML models
was recognized as a significant limitation of neural networks
(NNs). In recent years, Lipschitz bounds has been used for training ML
models towards achieving some degree of
robustness~\cite{hein-nips17,daniel-iclr18,scaman-nips18,pappas-nips19,dimakis-nips20,kornblith-icml21,kolter-nips21,allgower-csl22,wang-nips22,lomuscio-cvpr23,hu-nips23}. 

\subsubsection{Additional notation.}
%
An explanation problem is a tuple $\fml{E}=(\fml{M},(\mbf{v},q))$,
where $\fml{M}$ can either be a classification or a regression model,
and $(\mbf{v},q)$ is a target sample, with $\mbf{v}\in\mbb{F}$. 
For example, for the running example in~\cref{ex:tr:rt},
given the target sample $(\mbf{v}_1,q_1) = ((1,1),1)$, 
we can define an explanation problem $\fml{E}_{1}=(\fml{M}_{1},(\mbf{v}_1,q_1))$.

Given $\mbf{x},\mbf{v}\in\mbb{F}$, and $\fml{S}\subseteq\fml{F}$, the
predicate $\mbf{x}_{\fml{S}}=\mbf{v}_{\fml{S}}$ is defined as follows:
\[
\mbf{x}_{\fml{S}}=\mbf{v}_{\fml{S}} ~~ := ~~ \left(\bigwedge\nolimits_{i\in\fml{S}}x_i=v_i\right).
\]
The set of points for which $\mbf{x}_{\fml{S}}=\mbf{v}_{\fml{S}}$ is
defined by
$\Upsilon(\fml{S};\mbf{v})=\{\mbf{x}\in\mbb{F}\,|\,\mbf{x}_{\fml{S}}=\mbf{v}_{\fml{S}}\}$.

\subsubsection{Distributions, expected value.}
Throughout the paper, it is assumed a \emph{uniform probability
distribution} on features $\fml{F}$, and such that all features are independent.
The \emph{expected value} of an ML model $\tau:\mbb{F}\to\mbb{T}$
is denoted by $\mbf{E}[\tau]$. 
Furthermore, let
$\exv[\tau(\mbf{x})\,|\,\mbf{x}_{\fml{S}}=\mbf{v}_{\fml{S}}]$
represent the expected of $\tau$ over points in feature space
consistent with the coordinates of $\mbf{v}$ dictated by $\fml{S}$.
For discrete-valued features,
$\exv[\tau(\mbf{x})\,|\,\mbf{x}_{\fml{S}}=\mbf{v}_{\fml{S}}]$ is as
follows:
\begin{equation} \label{eq:evdef}
  \exv[\tau(\mbf{x})\,|\,\mbf{x}_{\fml{S}}=\mbf{v}_{\fml{S}}]
  :=\sfrac{1}{|\Upsilon(\fml{S};\mbf{v})|}
  \sum\nolimits_{\mbf{x}\in\Upsilon(\fml{S};\mbf{v})}\tau(\mbf{x}).
\end{equation}

\jnoteF{%
  Introduce definitions of expected value for continuous-valued
  features.
}

For real-valued features,
$\exv[\tau(\mbf{x})\,|\,\mbf{x}_{\fml{S}}=\mbf{v}_{\fml{S}}]$ is as
follows:
\begin{equation} \label{eq:evdefint}
  \exv[\tau(\mbf{x})\,|\,\mbf{x}_{\fml{S}}=\mbf{v}_{\fml{S}}]
  :=\sfrac{1}{|\Upsilon(\fml{S};\mbf{v})|}
  \int_{\Upsilon(\fml{S};\mbf{v})}\tau(\mbf{x})d\mbf{x}.
\end{equation}

\jnoteF{
  Given $\mbf{z}\in\mbb{F}$ and$\fml{S}\subseteq\fml{F}$, let
  $\mbf{z}_{\fml{S}}$ represent the vector composed of the coordinates
  of $\mbf{z}$ dictated by $\fml{S}$.
  \[
  \exv{\kappa\,|\,\mbf{x}_{\fml{S}}=\mbf{v}_{\fml{S}}}
  :=\frac{1}{|\Upsilon(\fml{S};\mbf{v})|}
  \sum\nolimits_{\mbf{x}\in\Upsilon(\fml{S};\mbf{v})}\kappa(\mbf{x})
  \]
}

\subsubsection{Shapley values \& SHAP scores.}
%
Shapley values were proposed in the context of game theory in the
early 1950s by L.\ S.\ Shapley~\cite{shapley-ctg53}. Shapley values
were defined given some set $\fml{S}$, and a \emph{characteristic
function}, i.e.\ a real-valued function defined on the subsets of
$\fml{S}$, $\cf:2^{\fml{S}}\to\mbb{R}$.%
\footnote{%
The original formulation also required super-additivity of the
characteristic function, but that condition has been relaxed in more
recent works~\cite{dubey-ijgt75,young-ijgt85}.}.
It is well-known that Shapley values represent the \emph{unique}
function that, given $\fml{S}$ and $\cf$, respects a number of
important axioms. More detail about Shapley values is available in
standard
references~\cite{shapley-ctg53,dubey-ijgt75,young-ijgt85,roth-bk88}.

In the context of explainability, Shapley values are most often
referred to as SHAP scores%
~\cite{kononenko-jmlr10,kononenko-kis14,lundberg-nips17,barcelo-aaai21,barcelo-jmlr23},
and consider a specific characteristic function
$\cf_e:2^{\fml{F}}\to\mbb{R}$,
which is defined by,
\begin{equation} \label{eq:cfs}
  \cf_e(\fml{S};\fml{E}) :=
  \exv[\tau(\mbf{x})\,|\,\mbf{x}_{\fml{S}}=\mbf{v}_{\fml{S}}].
\end{equation}
Thus, given a set $\fml{S}$ of features,
$\cf_e(\fml{S};\fml{E})$ represents the \emph{e}xpected value
of the classifier over the points of feature space represented by
$\Upsilon(\fml{S};\mbf{v})$.
%

\begin{table}[t]
  \centering
  \scalebox{1.0}{
  \renewcommand{\tabcolsep}{0.75em}
  \begin{tabular}{ccc} \toprule
    $\fml{S}$ & $\msf{rows}(\fml{S})$ & $\cfn{e}(\fml{S})$
    \\ \toprule
    $\emptyset$ & $1,2,3,4$ & $1-\alpha$ \\
    $\{1\}$ & $3,4$ & $1$ \\
    $\{2\}$ & $2,4$ & $1+\alpha$ \\
    $\{1,2\}$ & $4$ & $1$
    \\ \bottomrule
  \end{tabular}
}

  \caption{Expected values of $\tau_1$ for all possible sets $\fml{S}$
    of fixed features, given the target sample $((1,1),1)$.}
  \label{tab:avg01}
\end{table}

\begin{example}(Expected values for $\tau_1$.)
  For the explanation problem $\fml{E}_{1}$,
  the expected values of $\tau_1$ for all possible sets
  $\fml{S}$ of fixed features are shown in~\cref{tab:avg01}.
\end{example}

The formulation presented in earlier
work~\citep{barcelo-aaai21,barcelo-jmlr23} allows for different input
distributions when computing the expected values. For the purposes of
this paper, it suffices to consider solely a uniform input
distribution, and so the dependency on the input distribution is not
accounted for.
Independently of the distribution considered, it should be clear that
in most cases $\cfn{e}(\emptyset)\not=0$; this is the case for example
with boolean classifiers~\cite{barcelo-aaai21,barcelo-jmlr23}.

To simplify the notation, the following definitions are used,
\begin{align}
  \Delta_i(\fml{S};\fml{E},\cf) & :=
  \left(\cf(\fml{S}\cup\{i\})-\cf(\fml{S})\right),
  \label{eq:def:delta}
  \\[2pt]
  \varsigma(\fml{S}) & :=
  \sfrac{|\fml{S}|!(|\fml{F}|-|\fml{S}|-1)!}{|\fml{F}|!},
  \label{eq:def:vsigma}
\end{align}
(Observe that $\Delta_i$ is parameterized on $\fml{E}$ and $\cf$.)

Finally, let $\svn{E}:\fml{F}\to\mbb{R}$, i.e.\ the SHAP score for
feature $i$, be defined by,\footnote{%
Throughout the paper, the definitions of $\Delta_i$ and $\sv$ are
explicitly associated with the characteristic function used in their
definition.}.
\begin{equation} \label{eq:sv}
  \svn{E}(i;\fml{E},\cfn{e}):=\sum\nolimits_{\fml{S}\subseteq(\fml{F}\setminus\{i\})}\varsigma(\fml{S})\times\Delta_i(\fml{S};\fml{E},\cfn{e}).
\end{equation}
Given a sample $(\mbf{v},q)$, the SHAP score assigned to each
feature measures the \emph{contribution} of that feature with respect
to the prediction. 
From earlier work, it is understood that a positive/negative value
indicates that the feature can contribute to changing the prediction,
whereas a value of 0 indicates no
contribution~\citep{kononenko-jmlr10}.
%


\subsubsection{Related work.} 
SHAP scores, i.e.\ the use of Shapley values in XAI, are being used 
ubiquitously in a wide range of practical domains, especially through
the use of the tool SHAP~\cite{lundberg-nips17}.
The practical limitations of SHAP scores, e.g.\ the results obtained
with the tool SHAP, have been documented in the
literature~\cite{blockbaum-aistats20,najmi-icml20,friedler-icml20,friedler-nips21}.
The theoretical limitations of SHAP scores, i.e.\ issues resulting
from existing definitions of SHAP scores and obtained independently of
the tool SHAP, were reported in recent work~\cite{hms-corr23a}, but
targeting only classification models.
More recent work proposed solutions for addressing the theoretical
limitations of SHAP
scores~\cite{ignatiev-corr23a,izza-corr23,ignatiev-corr23b,izza-aaai24}. However,
the proposed solutions do not represent SHAP scores.

To the best of our knowledge, the identification of theoretical
limitations of SHAP scores in the case of regression models, including
regression models that respect Lipschitz continuity has not been
investigated.

\section{Formal Explainability \& Adversarial Examples} \label{sec:fxai}

We will opt to define a \emph{similarity} predicate, which will enable
us to abstract away the details of whether the ML model relates with
classification or regression.

\subsubsection{Similarity predicate.}
%
Given an ML model and some input $\mbf{x}$, the output of the ML model
is \emph{distinguishable} with respect to the sample $(\mbf{v},q)$ if the
observed change in the model's output is deemed sufficient; 
otherwise it is \emph{similar} (or indistinguishable).
This is represented by a \emph{similarity} predicate (which can be
viewed as a boolean function) 
$\similar:\mbb{F}\to\{\bot,\top\}$ (where $\bot$ signifies
\emph{false}, and $\top$ signifies \emph{true}).%
\footnote{%
For simplicity, and with a minor abuse of notation, when $\similar$
is used in a scalar context, it is interpreted as a boolean function,
i.e.\ $\similar:\mbb{F}\to\{0,1\}$, with 0 replacing $\bot$ and 1 
replacing $\top$.}
Concretely, given $\delta\in\mbb{R}$, $\similar(\mbf{x};\fml{E})$ holds true iff the change in
the ML model output is deemed \emph{insufficient} and so no observable
difference exists between the ML model's output for $\mbf{x}$ and
$\mbf{v}$ by a factor of $\delta$.
\footnote{
Throughout the paper, parameterizations are shown after the separator
';', and will be elided when clear from the context.}

For regression problems, given a change in the input from $\mbf{v}$ to
$\mbf{x}$, a change in the output is indistinguishable (i.e.\ the
outputs are similar) if,
\[
\similar(\mbf{x};\fml{E}) := [|\rho(\mbf{x})-\rho(\mbf{v})|\le\delta],
\]
otherwise, it is distinguishable.
\footnote{%
Exploiting a threshold to decide whether there exists an observable
change has been used in the context of adversarial
robustness~\cite{barrett-nips23}. Furthermore, the relationship
between adversarial examples and explanations is
well-known~\cite{inms-nips19,barrett-nips23}.}

For classification problems, similarity is defined to equate with not
changing the predicted class, in which case the parameter $\delta$ is always 0.
Given a change in the input from 
$\mbf{v}$ to $\mbf{x}$, a change in the output is indistinguishable
(i.e.\ the outputs are similar) if,
\[
\similar(\mbf{x};\fml{E}) := [|\kappa(\mbf{x})-\kappa(\mbf{v})|\le0],
\]
otherwise, it is distinguishable. Alternatively, it can be represented as
\[
\similar(\mbf{x};\fml{E}) := [\kappa(\mbf{x})=\kappa(\mbf{v})].
\]
(As shown in the remainder of this
paper, $\similar$ allows abstracting away whether the the underlying
model implements classification or regression.)

In the remainder of the paper, we will be computing the expected value
of the similarity predicate $\sigma$. In those situations, the
similarity predicate will be interpreted as a boolean function, where
$\bot$ corresponds to 0, and $\top$ corresponds to 1. (Clearly, we
could use additional notation to avoid this minor abuse of notation,
but opt instead to keep the notation as simple as possible.)
When $\sigma$ is interpreted as a boolean function it can be viewed as
a regression model that predicts one of two possible values, i.e.\ 0
and 1; hence, we can compute the expected value of $\sigma$ using
either~\eqref{eq:evdef} or~\eqref{eq:evdefint}.

\subsubsection{Adversarial examples.}
%
Adversarial examples serve to reveal the brittleness of ML
models~\cite{szegedy-iclr14,szegedy-iclr15}. Adversarial robustness
indicates the absence of adversarial examples. The importance of
deciding adversarial robustness is illustrated by a wealth of
competing alternatives~\cite{johnson-sttt23}.

Given a sample $(\mbf{v},q)$, and a norm $l_p$, a point
$\mbf{x}\in\mbb{F}$ is an \emph{adversarial example} (AEx) if the prediction
for $\mbf{x}$ is distinguishable from that for $\mbf{v}$. Formally, we
write,
\[
\aex(\mbf{x};\fml{E}) :=
\left(||\mbf{x}-\mbf{v}||_{l_p}\le\epsilon\right)\land
\neg\similar(\mbf{x};\fml{E})
\]
where the $l_p$ distance between the given point $\mbf{v}$ and other
points of interest is restricted to $\epsilon>0$.
Moreover, we define a \emph{constrained} adversarial example, such
that the allowed set of points is given by the predicate
$\mbf{x}_{\fml{S}}=\mbf{v}_{\fml{S}}$. Thus,
\[
\aex(\mbf{x},\fml{S};\fml{E}) :=
\left(||\mbf{x}-\mbf{v}||_{l_p}\le\epsilon\right)\land
\left(\mbf{x}_{\fml{S}}=\mbf{v}_{\fml{S}}\right)\land
\neg\similar(\mbf{x};\fml{E})
\]
Additionally, we use $l_p$-minimal AExs to refer to AExs
having the minimal distance measured by the norm $l_p$
around the given point $\mbf{v}$.

\begin{example}(AExs for $\fml{E}_{1}$.)
  We consider the norm $l_0$, i.e.\ the Hamming distance.
  For the explanation problem $\fml{E}_{1}$,
  there is a distance 1 AEx, i.e.\ by changing feature
  1 to a value other than 1, the prediction changes to $1+2\alpha$,
  this assuming the value of feature 2 is also 1.
  Moreover, if feature 1 is fixed, i.e. $1\in\fml{S}$, then model does
  not have \emph{any} AEx.
  It is also plain that feature 2 does not occur is any $l_0$-minimal
  AEx for $\fml{E}_{1}$.
\end{example}

\subsubsection{Abductive and contrastive explanations.}
Abductive and contrastive explanations (AXps/CXps)
represent the two examples of formal explanations for
classification
problems~\cite{kutyniok-jair21,msi-aaai22,darwiche-lics23}. This paper
studies a generalized formulation that also encompasses regression
problems.

A weak abductive explanation (WAXp) denotes a set of features
$\fml{S}\subseteq\fml{F}$, such that for every point in feature space
the ML model output is \emph{similar} to the given sample: $(\mbf{v},q)$.
The condition for a set of features to represent a WAXp (which also
defines a corresponding predicate $\waxp$) is as follows:
\begin{equation}
  \waxp(\fml{S};\fml{E}) :=
  \exv[\similar(\mbf{x};\fml{E})\,|\,\mbf{x}_{\fml{S}}=\mbf{v}_{\fml{S}}]
  = 1.
\end{equation}
Moreover, an AXp is a subset-minimal WAXp, that is,
\begin{equation} \label{eq:axp}
\axp(\fml{S};\fml{E}) := \waxp(\fml{S};\fml{E})\land\forall(t\in\fml{S}).\neg\waxp(\fml{S}\setminus\{t\};\fml{E}).
\end{equation}

A weak contrastive explanation (WCXp) denotes a set of features
$\fml{S}\subseteq\fml{F}$, such that there exists some point in
feature space, where only the features in $\fml{S}$ are allowed to
change, that makes the ML model output distinguishable from the given
sample $(\mbf{v},q)$.
The condition for a set of features to represent a WCXp (which also
defines a corresponding predicate $\wcxp$) is as follows:
\begin{equation}
  \wcxp(\fml{S};\fml{E}) :=
  \exv[\similar(\mbf{x};\fml{E})\,|\,\mbf{x}_{\fml{F}\setminus\fml{S}}=\mbf{v}_{\fml{F}\setminus\fml{S}}]
  < 1.
\end{equation}
Moreover, a CXp is a subset-minimal WCXp, that is,
\begin{equation} \label{eq:cxp}
\cxp(\fml{S};\fml{E}) := \wcxp(\fml{S};\fml{E})\land\forall(t\in\fml{S}).\neg\wcxp(\fml{S}\setminus\{t\};\fml{E}).
\end{equation}

\begin{example}(AXps/CXps for $\fml{E}_{1}$.)
  For the explanation problem $\fml{E}_{1}$, let us first consider a
  classification model, i.e.\ all $\tau_1(\mbf{x})$ are integer values,
  each denoting a class. \\
  If feature 1 is fixed to value 1, then the prediction is 1, and so
  the expected value of the similarity predicate is also 1. Otherwise,
  if the value of feature 1 changes to a value other than 1, then the
  prediction is some value other than 1. Thus, the similarity
  predicate is not 1 on all points of feature space, and so its
  expected value is less than 1.
  As result, it is immediate that $\{1\}$ is one (and the only) AXp,
  whereas $\{1\}$ is the one (and also the only) CXp. \\
  Now let us consider that $\fml{M}_{1}$ implements a regression function,
  e.g.\ some $\tau_1(\mbf{x})$ are real value. In this case, we need to
  specify a value of $\delta$ for the similarity predicate. Clearly,
  the value of $\delta$ cannot be arbitrarily large; otherwise we
  would be unable to distinguish 1 from the other values. For example,
  for $\alpha>0$, we can pick a value of $\delta$ no larger than
  $\alpha$. \\
  In this case, the sets of AXps and CXps remains unchanged.
  The computation of AXps/CXps for the explanation problem $\fml{E}_{1}$
  are summarized in~\cref{tab:xp01}.
\end{example}

\jnoteF{Show table summarizing the computation of AXps/CXps.}

\begin{table}[t]
\centering

\begin{tabular}{ccccc} \toprule
$\fml{S}$ & $\fml{F}\setminus\fml{S}$ & $\msf{rows}(\fml{S})$ & $\waxp(\fml{S})?$ & $\wcxp(\fml{F}\setminus\fml{S})?$
\\ \toprule
$\emptyset$ & $\{1,2\}$ & $1,2,3,4$ & No & Yes \\
$\{1\}$ & $\{2\}$ & $3,4$ & Yes & No \\
$\{2\}$ & $\{1\}$ & $2,4$ & No & Yes \\
$\{1,2\}$ & $\emptyset$ & $4$ & Yes & No
\\ \bottomrule
\end{tabular}

\caption{WAXps/WCXps of $\fml{E}_{1}$ for all possible sets $\fml{S}$
of fixed features.}
\label{tab:xp01}
\end{table}

One can prove that a set of features is an AXp iff it is a minimal
hitting set of the set of CXps, and vice-versa~\cite{msi-aaai22}.
(Although this result has been proved for classification problems, the
use of the similarity predicate generalizes the result also to
regression problems.)



By examining the definitions of AEx and WCXp, it is straightforward to
prove that there exists a constrained
AEx with the features $\fml{F}\setminus\fml{S}$ iff
the set $\fml{S}$ is a weak CXp.

\paragraph{Feature (ir)relevancy.}
%
The set of features that are included in at least one (abductive) 
explanation are defined as follows:
\begin{equation}
  \mathfrak{F}(\fml{E}):=\{i\in\fml{X}\,|\,\fml{X}\in2^{\fml{F}}\land\axp(\fml{X})\},
\end{equation}
where predicate $\axp(\fml{X})$ holds true iff $\fml{X}$ is an AXp.
(A well-known result is that $\mathfrak{F}(\fml{E})$ remains unchanged
if CXps are used instead of AXps~\cite{msi-aaai22}.) 
Finally, a feature $i\in\fml{F}$ is \emph{irrelevant}, i.e.\ predicate
$\irrelevant(i)$ holds true, if $i\not\in\mathfrak{F}(\fml{E})$;
otherwise feature $i$ is \emph{relevant}, and predicate $\relevant(i)$
holds true. 
Clearly, given some explanation problem $\fml{E}$,
$\forall(i\in\fml{F}).\irrelevant(i)\leftrightarrow\neg\relevant(i)$.

\begin{example}((Ir)relevant features for $\fml{E}_{1}$.)
  Given that $\{\{1\}\}$ denotes the set of AXps (and of CXps), it is
  immediate that feature 1 is relevant, and feature 2 is irrelevant.
\end{example}

\section{New Limitations of SHAP Scores}
\label{sec:issues}

The recent proposal of polynomial-time algorithms for the exact
computation of SHAP
scores~\cite{barcelo-aaai21,vandenbroeck-aaai21,vandenbroeck-jair22,barcelo-jmlr23}
has facilitated their more rigorous assessment.
Building on those efficient algorithms for computing SHAP scores,
recent work uncovered examples of classifiers for which the computed
(exact) scores are thoroughly unsatisfactory~\cite{hms-corr23a}.

\begin{table}[t]
  \begin{tabular}{cccccc} \toprule
  \multicolumn{6}{c}{$i=1$} \\
  \toprule
  $\fml{S}$ & $\cf_e(\fml{S})$ & $\cf_e(\fml{S}\cup\{1\})$ &
  $\Delta_1(\fml{S})$ & $\varsigma(\fml{S})$ &
  $\varsigma(\fml{S})\times\Delta_1(\fml{S})$ \\
  \toprule
  $\emptyset$ & $1-\alpha$ & 1 & $\alpha$ & $\nfrac{1}{2}$ &
  $\nfrac{\alpha}{2}$ \\
  $\{2\}$ & $1+\alpha$ & 1 & $-\alpha$ & $\nfrac{1}{2}$ &
  $-\nfrac{\alpha}{2}$ \\
  \midrule
  \multicolumn{5}{r}{$\svn{E}(1)~~=$} & \multicolumn{1}{c}{0} \\
  \toprule
  \multicolumn{6}{c}{$i=2$} \\
  \toprule
  $\fml{S}$ & $\cfn{e}(\fml{S})$ & $\cfn{e}(\fml{S}\cup\{2\})$ &
  $\Delta_2(\fml{S})$ & $\varsigma(\fml{S})$ &
  $\varsigma(\fml{S})\times\Delta_2(\fml{S})$ \\
  \toprule
  $\emptyset$ & $1-\alpha$ & $1+\alpha$ & $2\alpha$ & $\nfrac{1}{2}$
  & $\alpha$ \\
  $\{1\}$ & 1 & 1 & 0 & $\nfrac{1}{2}$ & 0 \\
  \midrule
  \multicolumn{5}{r}{$\svn{E}(2)~~=$} & \multicolumn{1}{c}{$\alpha$} \\
  \bottomrule
\end{tabular}

  \caption{Computation of SHAP scores for $\fml{E}_{1}$.}
  \label{tab:svs01}
\end{table}


\subsection{Classification -- Boolean Domains} \label{ssec:bresb}

This section proves that for arbitrary large numbers of variables,
there exist boolean functions and samples for which the SHAP scores
exhibit the issues reported in recent work~\citep{hms-corr23a}.
The detailed proofs are included in the supplemental materials.
Given a classifier $\fml{M}$, with sample $(\mbf{v},c)$, and with
$i,i_1,i_2\in\fml{F}$, the issues are shown in~\cref{ijar:issues:tab}.
\begin{table*}[t]
\centering
\begin{tabular}{cl} \toprule
Issue & Condition \\
\toprule
I1 &
$\irrelevant(i)\land\left(\svn{E}(i)\not=0\right)$
\\[2.5pt]
I2 &
$\irrelevant(i_1)\land\relevant(i_2)\land\left(|\svn{E}(i_1)|>|\svn{E}(i_2)|\right)$
\\[2.5pt]
I3 &
$\relevant(i)\land\left(\svn{E}(i)=0\right)$
\\[2.5pt]
I4 &
$[\irrelevant(i_1)\land\left(\svn{E}(i_1)\not=0\right)]\land[\relevant(i_2)\land\left(\svn{E}(i_2)=0\right)]$
\\[2.5pt]
I5 &
$[\irrelevant(i)\land\forall_{1\le{j}\le{m},j\not=i}\left(|\svn{E}(j)|<|\svn{E}(i)|\right)]$
\\[2.5pt]
I6 &
$[\irrelevant(i_1)\land\relevant(i_2)\land(\svn{E}(i_1)\times\svn{E}(i_2)>0)]$
\\[2.5pt]
\bottomrule
\end{tabular}
\caption{Issues with SHAP scores. Free variables, $i,i_1,i_2$ are
quantified existentially over the set of features
$\fml{F}$. The occurrence of some issues implies the occurrence of other issues.
I2 is implied by the occurrence I4 and I5.}
\label{ijar:issues:tab}
\end{table*}

Throughout this section, let $m$ be the number of variables of
the functions we start from, and let $n$ denote the number of
variables of the functions we will be constructing. In this case, we
set $\fml{F}=\{1,\ldots,n\}$.
Furthermore, for the sake of simplicity, 
we opt to introduce the new features as the last features (e.g., feature $n$).
Besides, we opt to set the values of these additional features to 1 in the sample $(\mbf{v}, c)$
that we intend to explain, that is, $v_n=1$.
When considering two additional features, feature $n$ and feature $n-1$,
we will assume that feature $n$ is irrelevant while feature $n-1$ is relevant.
This choice does not affect the proof's argument in any way.

For a boolean function $\kappa$,
we use $\kappa_0$ to denote the conditioning of the function $\kappa$ on $x_n=0$ (i.e. $\kappa|_{x_n=0}$),
and $\kappa_1$ to denote the conditioning of the function $\kappa$ on $x_n=1$.
Besides, we use $\kappa_{00}$ to denote the conditioning of the function $\kappa$ on $x_n=0$ and $x_{n-1}=0$ (i.e. $\kappa|_{x_n=0,x_{n-1}=0}$),
$\kappa_{01}$ for the conditioning on $x_n=0$ and $x_{n-1}=1$,
$\kappa_{10}$ for the conditioning on $x_n=1$ and $x_{n-1}=0$,
and $\kappa_{11}$ for the conditioning on $x_n=1$ and $x_{n-1}=1$.
Moreover, under a uniform input distribution, the following equations hold in general.
\begin{align} \label{eq:exvdec01}
\exv[\kappa] &= \frac{1}{2}\cdot\exv[\kappa_0] + \frac{1}{2}\cdot\exv[\kappa_1],
\end{align}
\begin{align} \label{eq:exvdec02}
\Delta_{n}(\fml{S};\fml{E},\cfn{e})
= \frac{1}{2} \cdot ( \exv[\kappa_1|\mbf{x}_{\fml{S}}=\mbf{v}_{\fml{S}}] - \exv[\kappa_0|\mbf{x}_{\fml{S}}=\mbf{v}_{\fml{S}}] ).
\end{align}
By choosing different functions for $\kappa_{00}$, $\kappa_{01}$,
$\kappa_{10}$ and $\kappa_{11}$, we are able to construct functions
$\kappa$ exhibiting the issues reported in~\cref{ijar:issues:tab}.

\begin{restatable}{proposition}{PropIRR}%
  \label{prop:irr}%
  For any $n\ge3$, there exist boolean functions defined on $n$
  variables, and at least one sample, which exhibit an
  issue~I1, i.e. there exists an irrelevant
  feature $i\in\fml{F}$, such that $\svn{E}(i)\not=0$.
\end{restatable}

\begin{restatable}{proposition}{PropREL}%
  \label{prop:rel}%
  For any odd $n\ge3$, there exist boolean functions defined on $n$
  variables, and at least one sample, which exhibits an
  issue~I3, i.e. there exists a
  relevant feature $i\in\fml{F}$, such that $\svn{E}(i)=0$.
\end{restatable}

\begin{restatable}{proposition}{PropDisorder}%
  \label{prop:disorder}%
  For any even $n\ge4$, there exist boolean functions defined on $n$
  variables, and at least one sample, which exhibits an
  issue~I4, i.e. there exists an
  irrelevant feature $i_1\in\fml{F}$, such that $\svn{E}(i_1)\neq0$, and a
  relevant feature $i_2\in\fml{F}\setminus\{i_1\}$, such that
  $\svn{E}(i_2)=0$.
\end{restatable}

\begin{restatable}{proposition}{PropHighest}%
  \label{prop:highest}%
  For any $n\ge4$, there exists boolean functions defined on $n$
  variables, and at least one sample, which exhibits an
  issue~I5, i.e. there exists an
  irrelevant feature $i\in\fml{F}$, such that
  $|\svn{E}(i)|=\max\{|\svn{E}(j)|\:|\,j\in\fml{F}\}$.
\end{restatable}

\begin{restatable}{proposition}{PropSign}%
  \label{prop:sign}
  For any $n\ge4$,
  there exist functions defined on $n$
  variables, and at least one sample, which exhibits an
  issue~I6, i.e. there exists an
  irrelevant feature $i_1\in\fml{F}$, and a relevant feature $i_2\in\fml{F}\setminus\{i_1\}$,
  such that $\svn{E}(i_1) \times \svn{E}(i_2) > 0$.
\end{restatable}

\begin{restatable}{proposition}{PropEvenly}%
\label{prop:evenly}
  For boolean classifiers,
  if $\fml{E} = (\fml{M}, (\mbf{v}, c))$ exhibits an identified issue (I1 to I6),
  so does $\fml{E}' = (\fml{M}', (\mbf{v}, \neg c))$, where $\fml{M}'$ is the negated classifier of $\fml{M}$.
\end{restatable}

The implication of~\cref{prop:evenly} is that all the identified issues distributed evenly
for samples where the prediction takes value 1 and samples where the prediction takes value 0.

\subsection{Classification -- Real-valued Domains} \label{ssec:bresr}
Evidently, equations \eqref{eq:exvdec01} and \eqref{eq:exvdec02} can be extended
to discrete domains and real-valued domains.
The difference lies in the computation of $\exv[\kappa]$, 
which can be done either through counting (when the domains are discrete)
or integration (when the domains are real-valued).

\subsection{Regression -- Finite Codomain}

We illustrate variants of the known issues of SHAP scores with the
explanation problem $\fml{E}_{1}$. For an arbitrary $\alpha\not=0$,
\cref{tab:svs01} summarizes the computation of the SHAP scores. As can
be concluded, the SHAP score of feature 1 is always 0, and the SHAP
score of feature 2 is not 0 (assuming $\alpha\not=0$).
For example, if we pick $\alpha=1$, we get a (numeric) classification
problem with $\fml{K}=\{-5,1,3\}$.
However, we can pick $\alpha=\sfrac{1}{4}$, to obtain a regression
problem with codomain $\{-\sfrac{1}{2},1,\sfrac{3}{2}\}$.
In both cases, it is the case that feature 1, which is a relevant
feature, has a SHAP score of 0, and feature 2, which is an irrelevant
feature, has a non-zero SHAP score.

Although the examples above capture both classification and
regression, one can argue that these examples are still somewhat
artificial, since the codomain of the regression model consists of a
small number of different values, either integer or real-valued.
%
The next section shows how this criticism can be addressed.

%
%
%
%
\subsection{Regression -- Uncountable Codomain}

We now study a regression model containing a non-finite (in fact
uncountable) number of values in the codomain of its regression
function.

%
%

%

\begin{example}(Regression model $\fml{M}_2$.) \label{ex:rm02}
  We consider a regression problem defined over two real-valued
  features, taking values from interval
  $[-\sfrac{1}{2},\sfrac{3}{2}]$.
  Thus, we have $\fml{F}=\{1,2\}$,
  $\mbb{D}_1=\mbb{D}_2=\mbb{D}=[-\sfrac{1}{2},\sfrac{3}{2}]$,
  $\mbb{F}=\mbb{D}\times\mbb{D}$.
  (We also let $\mbb{D}^{+}=[\sfrac{1}{2},\sfrac{3}{2}]$
  and $\mbb{D}^{-}=\mbb{D}\setminus\mbb{D}^{+}$.) 
  In addition, the regression model maps to real values,
  i.e.\ $\mbb{K}=\mbb{R}$, and is defined as follows:
  \[
    \rho_2(x_1,x_2) =
    \left\{
    \begin{array}{lcl}
      x_1 & ~~ & \tn{if $x_1\in\mbb{D}^{+}$} \\ 
      x_2-2 & ~~ & \tn{if
        $x_1\not\in\mbb{D}^{+}\land{x_2}\not\in\mbb{D}^{+}$}\\
      x_2+1 & ~~ & \tn{if
        $x_1\not\in\mbb{D}^{+}\land{x_2}\in\mbb{D}^{+}$}\\
    \end{array}
    \right.
  \]
  As a result, the regression model is represented by
  $\fml{M}_2=(\fml{F},\mbb{F},\mbb{K},\rho_2)$.
  Moreover, we assume the target sample to be $(\mbf{v}_2,q_2)=((1,1),1)$,
  and so the explanation problem becomes
  $\fml{E}_2=(\fml{M}_2,(\mbf{v}_2,q_2))$.
\end{example}

\begin{table}[t]
  \centering
  \renewcommand{\tabcolsep}{0.775em}
  \begin{tabular}{ccccc} \toprule
    $\fml{S}$ & $\emptyset$ & $\{1\}$ & $\{2\}$ & $\{1,2\}$ \\
    \midrule[0.75pt]
    $\exv[\rho_2(\mbf{x})\,|\,\mbf{x}_{\fml{S}}=\mbf{v}_{\fml{S}}]$
    & $\sfrac{1}{2}$ & 1 & $\sfrac{3}{2}$ & 1 \\
    \bottomrule
  \end{tabular}
  \caption{Expected values of $\rho_2$, for each possible set
    $\fml{S}$ of fixed features, and given the sample $((1,1),1)$.}
  \label{tab:avg02}
\end{table}

\begin{example}(AXps, CXps and AExs for $\fml{E}_2$.)
  Given the regression model $\fml{M}_2$, we define the similarity
  predicate by picking a suitably small value $\delta$, e.g.\ 
  $\delta<\sfrac{1}{4}$.
  This suffices to ensure that $\sigma$ only takes value 1 when
  feature 1 takes value 1.\\
  Given the above, the similarity predicate takes value 1 only when
  feature 1 takes value 1, and independently of the value assigned to
  feature 2.\\
  Thus, fixing feature 1 ensures that the similarity predicate is 1,
  and so the expected value is 1. Otherwise, if feature 1 is allowed
  to take a value other than 1, then the similarity predicate can take
  value 0, and so the expected value is no longer 1.
  As a result, $\{1\}$ is one (and the only) AXp, and $\{1\}$ is also
  one (and the only) CXp.\\
  A similar analysis allow concluding that a $l_0$-minimal AEx
  exists iff feature 1 is allowed to take any value from
  its domain.
\end{example}

\begin{example}(SHAP scores for $\fml{E}_2$.)
  The expected values of $\rho_2$ for all possible sets $\fml{S}$ of
  fixed features is shown in~\cref{tab:avg02}.%
  \footnote{%
  The computation of these expected values is fairly straightforward,
  and is summarized in the supplemental materials.}
  Observe that these values correspond exactly to the expected values
  shown in~\cref{tab:avg01} when $\alpha=\sfrac{1}{2}$.
  Hence, the computation of SHAP scores is the one shown
  in~\cref{tab:svs01} by setting $\alpha=\sfrac{1}{2}$, and so
  $\svn{E}(1)=0$ and $\svn{E}(2)=\sfrac{1}{2}$.
\end{example}

The above examples demonstrate that one can construct a regression
model for which the computed SHAP scores will be misleading, assigning
importance to a feature having no influence in the prediction, and
assigning no important to a feature having complete influence in the
prediction.

Given the above, this section proves the following result.
\begin{proposition}
  There exist regression models, with an uncountable codomain, for
  which each feature $i$ is either irrelevant and its SHAP score is
  non-zero or the feature is relevant and its SHAP score is zero.
\end{proposition}

Nevertheless, the previous example is also open to criticism because
the regression function is not continuous.
Continuity is an important aspect of some ML models, such that in
recent years Lipschitz continuity has been exploited to ensure the
adversarial robustness of ML models.
Thus, one additional challenge is whether regression models respecting
Lipschitz continuity (and so plain continuity) can produce
unsatisfactory SHAP scores. The next section addresses this challenge.

\begin{figure*}[t]
  \centering
  \begin{tabular}{l}
    $  \rho_3(x_1,x_2) =
      \left\{
      \begin{array}{lcl}
        x_1 &
        \quad & \tn{if $x_2\le1\land\alpha{x_1}\le\alpha$} \\[2pt]
        (1+4|\alpha|)x_1-4|\alpha| &
        \quad & \tn{if $x_2\le1\land\alpha{x_1}\ge\alpha$} \\[2pt]
        28|\alpha|{x_1}{x_2} + (1-28|\alpha|)x_1 -28|\alpha|{x_2}+28|\alpha| &
        \quad & \tn{if $x_2\ge1\land\alpha{x_1}\le\alpha$} \\[2pt]
        -4|\alpha|{x_1}{x_2}+(1+8|\alpha|)x_1+4|\alpha|{x_2}-8|\alpha| &
        \quad & \tn{if $x_2\ge1\land\alpha{x_1}\ge\alpha$} \\[5pt]
      \end{array}
      \right.
      $
  \end{tabular}
  \caption{Example of regression model that is Lipschitz continuous.}
  \label{fig:rmlc}
\end{figure*}
%
%

\subsection{Regression -- Lipschitz Continuity}
\label{sec:lc}


The examples studied in the earlier sections expand significantly the
range of ML models for which unsatisfactory can be produced.
However, a major source of criticism is that these ML models are not
continuous.

We now discuss a \emph{family} of continuous ML models, that
also produce unsatisfactory SHAP scores. We will then argue that the
proposed family of ML models also ensures Lipschitz continuity.

\begin{example} \label{ex:rm03}
  We consider a regression problem defined over two real-valued
  features, taking values from interval $[0,2]$.
  Thus, we have $\fml{F}=\{1,2\}$,
  $\mbb{D}_1=\mbb{D}_2=\mbb{D}=[0,2]$,
  $\mbb{F}=\mbb{D}\times\mbb{D}$.
  In addition, the regression model maps to real values,
  i.e.\ $\mbb{K}=\mbb{R}$, and is defined as shown in~\cref{fig:rmlc}.
  The value is such that $\alpha\in\mbb{R}\setminus\{0\}$.
  As a result, the regression model is represented by
  $\fml{M}_3=(\fml{F},\mbb{F},\mbb{K},\rho_3)$.
  Moreover, we assume the target sample to be
  $(\mbf{v}_3,q_3)=((1,1),1)$, and so the explanation problem becomes
  $\fml{E}_3=(\fml{M},(\mbf{v}_3,q_3))$.
\end{example}

By inspection, it is plain that $\rho_3$ is continuous; this is
further discussed below.

\begin{example}(AXps, CXps and AExs for $\fml{E}_3$.)
  As before, it is plain to reach the conclusion that the set of AXps
  is $\{\{1\}\}$, and this is also the set of CXps. Moreover, and as
  before, there is a $l_0$-minimal AEx containing feature 1.
\end{example}

\begin{example}(SHAP scores for $\fml{E}_3$.)
  The regression model $\fml{M}_3$ is devised such that the expected
  values of $\rho_3$ for each possible set $\fml{S}$ of 
  fixed features are exactly the ones shown in~\cref{tab:avg01}.
  As a result, the computed SHAP scores are the same as before, and so
  they are again unsatisfactory.
\end{example}

Given the above, this section proves the following result.
\begin{proposition}
  There exist regression models, respecting Lipschitz continuity, for
  which each feature $i$ is either irrelevant and its SHAP score is
  non-zero or the feature is relevant and its SHAP score is zero.
\end{proposition}

Finally, it is simple to prove that $\rho_3$ is Lipschitz-continuous.
The proof is included in the supplemental materials. Also, the
continuity of $\rho_3$ (claimed above) is implied by the fact that
$\rho_3$ Lipschitz-continuous.

%

\subsection{Regression -- Arbitrary Differentiability}
\label{ssec:blc}

This section argues that the issues with SHAP scores can be identified
even when regression models are arbitrary differentiable.%
\footnote{
It is well-known that differentiability does not imply Lipschitz
continuity. Hence, we consider the two cases separately.}
We provide a simple argument below.
\footnote{The actual construction of the regression model becomes
somewhat more cumbersome, and so we just give the rationale for
constructing the model.}

To devise an arbitrary differentiable regression model, we use
$\rho_3$ (see~\cref{fig:rmlc}) and the guiding example.
The main insight is to take the rectangle
$[1-\epsilon,1+\epsilon]×[0,2]$ with (a very small) $\epsilon>0$.
In this rectangle, we replace the function by polunomials on $x_1$
with the same value on $\{1\}×[0,2]$, on $\{1-\epsilon\}×[0,2]$ and on
$\{1+\epsilon\}×[0,2]$ (the dependence on $x_2$ remains unchanged, we
still have a polynomial on $x_1$ for $x_2\in[0,1]$, another one for
$x_2=2$ and the polynomial for $x_2\in[1,2]$ is given by $2-x_2$ times
the former plus $x_2-1$ times the latter like for left and right
sides) plus we also fix the value of the $n$ first derivatives 
on $x_1$ for $x_1\in\{1-\epsilon,1+\epsilon\}$. As these are $3+2n$
(with $n$ being the number of times the counter-example will be
derivable) constraints, it is possible with infinitly many $3+2n$
degree polynomials (exactly one of them will be degree $2+2n$ or less
but it is not necessarly the one that interest us) so we can do it for
every $n$. Now in these infinitly many polynomials, we pick the one
with the same average value on the rectangle of the original function,
we can do it because we only fixed part of relative size 0 so in these
polynomials there are every average value in $\mbb{R}$ (we can also fix
the fact that $\rho>1$ if $x_1>1$ and $\rho<1$ if $x_1<1$ with 
higher non-controlled degree if necessary). The same can be done with
the other non-derivable line as the conditions are the same. Then we
have a new function with n-times derivability and lipschitz that still
have the same properties. 

Given the above, this section proves the following result.
\begin{proposition}
  There exist regression models, that are arbitrarily differentiable,
  for which each feature $i$ is either irrelevant and its SHAP score is
  non-zero or the feature is relevant and its SHAP score is zero.
\end{proposition}


\section{Conclusions} \label{sec:conc}

SHAP scores find an ever-increasing range of uses in XAI. However,
recent work demonstrated that the exact computation of SHAP scores can
yield thoroughly unsatisfactory results~\cite{hms-corr23a}, in the
concrete case of ML classifiers. 
Building on this earlier work, this paper demonstrates that the
limitations identified in the case of ML classifiers can be reproduced
in far more general settings, that include regression models, but also
regression models that respect Lipschitz continuity.
As a result, even for ML models that respect Lipschitz continuity, and
so respect a well-known criterion for adversarial robustness,
SHAP scores are shown to be 
unsatisfactory.

The paper also argues that similar examples of regression models that
yield unsatisfactory results can be obtained for $C^{\infty}$
functions.
In light of earlier results, but also the new results demonstrated in
this paper, for the most widely used ML models there exist examples
for which computed SHAP scores will be entirely unsatisfactory.
This further justifies alternative measures of feature
importance~\cite{izza-aaai24}.
Finally, recent work have outline mechanisms for correcting SHAP
scores as well as relating them with other
alternatives~\cite{lhms-corr24,lhams-corr24}.

\newtoggle{mkbbl}

\settoggle{mkbbl}{false}

\iftoggle{mkbbl}{
    \bibliography{refs,arlc}

\begin{thebibliography}{48}
\providecommand{\natexlab}[1]{#1}

\bibitem[{Arenas et~al.(2021)Arenas, Barcel{\'{o}}, Bertossi, and
  Monet}]{barcelo-aaai21}
Arenas, M.; Barcel{\'{o}}, P.; Bertossi, L.~E.; and Monet, M. 2021.
\newblock The Tractability of {SHAP}-Score-Based Explanations for
  Classification over Deterministic and Decomposable Boolean Circuits.
\newblock In \emph{AAAI}, 6670--6678.

\bibitem[{Arenas et~al.(2023)Arenas, Barcel{\'{o}}, Bertossi, and
  Monet}]{barcelo-jmlr23}
Arenas, M.; Barcel{\'{o}}, P.; Bertossi, L.~E.; and Monet, M. 2023.
\newblock On the Complexity of {SHAP}-Score-Based Explanations: Tractability
  via Knowledge Compilation and Non-Approximability Results.
\newblock \emph{J. Mach. Learn. Res.}, 24: 63:1--63:58.

\bibitem[{Biradar et~al.(2023)Biradar, Izza, Lobo, Viswanathan, and
  Zick}]{izza-corr23}
Biradar, G.; Izza, Y.; Lobo, E.; Viswanathan, V.; and Zick, Y. 2023.
\newblock Axiomatic Aggregations of Abductive Explanations.
\newblock \emph{CoRR}, abs/2310.03131.

\bibitem[{Biradar et~al.(2024)Biradar, Izza, Lobo, Viswanathan, and
  Zick}]{izza-aaai24}
Biradar, G.; Izza, Y.; Lobo, E.; Viswanathan, V.; and Zick, Y. 2024.
\newblock Axiomatic Aggregations of Abductive Explanations.
\newblock In \emph{AAAI}, 11096--11104.

\bibitem[{Brix et~al.(2023)Brix, M{\"{u}}ller, Bak, Johnson, and
  Liu}]{johnson-sttt23}
Brix, C.; M{\"{u}}ller, M.~N.; Bak, S.; Johnson, T.~T.; and Liu, C. 2023.
\newblock First three years of the international verification of neural
  networks competition {(VNN-COMP)}.
\newblock \emph{Int. J. Softw. Tools Technol. Transf.}, 25(3): 329--339.

\bibitem[{Cranko et~al.(2021)Cranko, Shi, Zhang, Nock, and
  Kornblith}]{kornblith-icml21}
Cranko, Z.; Shi, Z.; Zhang, X.; Nock, R.; and Kornblith, S. 2021.
\newblock Generalised Lipschitz Regularisation Equals Distributional
  Robustness.
\newblock In \emph{ICML}, 2178--2188.

\bibitem[{Darwiche(2023)}]{darwiche-lics23}
Darwiche, A. 2023.
\newblock Logic for Explainable {AI}.
\newblock In \emph{{LICS}}, 1--11.

\bibitem[{Dubey(1975)}]{dubey-ijgt75}
Dubey, P. 1975.
\newblock On the uniqueness of the {S}hapley value.
\newblock \emph{International Journal of Game Theory}, 4: 131--139.

\bibitem[{{European Union}(2024)}]{euaict24}
{European Union}. 2024.
\newblock Artificial Intelligence Act.
\newblock \url{https://tinyurl.com/yj2zv67w}.

\bibitem[{Fazlyab et~al.(2019)Fazlyab, Robey, Hassani, Morari, and
  Pappas}]{pappas-nips19}
Fazlyab, M.; Robey, A.; Hassani, H.; Morari, M.; and Pappas, G.~J. 2019.
\newblock Efficient and Accurate Estimation of {L}ipschitz Constants for Deep
  Neural Networks.
\newblock In \emph{NeurIPS}, 11423--11434.

\bibitem[{Goodfellow, Shlens, and Szegedy(2015)}]{szegedy-iclr15}
Goodfellow, I.~J.; Shlens, J.; and Szegedy, C. 2015.
\newblock Explaining and Harnessing Adversarial Examples.
\newblock In \emph{ICLR}.

\bibitem[{Hanspal and Lomuscio(2023)}]{lomuscio-cvpr23}
Hanspal, H.; and Lomuscio, A. 2023.
\newblock Efficient Verification of Neural Networks Against LVM-Based
  Specifications.
\newblock In \emph{CVPR}, 3894--3903.

\bibitem[{Havens et~al.(2023)Havens, Araujo, Garg, Khorrami, and
  Hu}]{hu-nips23}
Havens, A.~J.; Araujo, A.; Garg, S.; Khorrami, F.; and Hu, B. 2023.
\newblock Exploiting Connections between {L}ipschitz Structures for Certifiably
  Robust Deep Equilibrium Models.
\newblock In \emph{NeurIPS}.

\bibitem[{Hein and Andriushchenko(2017)}]{hein-nips17}
Hein, M.; and Andriushchenko, M. 2017.
\newblock Formal Guarantees on the Robustness of a Classifier against
  Adversarial Manipulation.
\newblock In \emph{NeurIPS}, 2266--2276.

\bibitem[{{High-Level Expert Group on AI}(2019)}]{alegai19}
{High-Level Expert Group on AI}. 2019.
\newblock Ethics guidelines for trustworthy {AI}.
\newblock \url{https://tinyurl.com/yyzwmm79}.

\bibitem[{Huang and Marques{-}Silva(2023)}]{hms-corr23a}
Huang, X.; and Marques{-}Silva, J. 2023.
\newblock The Inadequacy of {S}hapley Values for Explainability.
\newblock \emph{CoRR}, abs/2302.08160.

\bibitem[{Huang et~al.(2021)Huang, Zhang, Shi, Kolter, and
  Anandkumar}]{kolter-nips21}
Huang, Y.; Zhang, H.; Shi, Y.; Kolter, J.~Z.; and Anandkumar, A. 2021.
\newblock Training Certifiably Robust Neural Networks with Efficient Local
  {L}ipschitz Bounds.
\newblock In \emph{NeurIPS}, 22745--22757.

\bibitem[{Ignatiev, Narodytska, and Marques{-}Silva(2019)}]{inms-nips19}
Ignatiev, A.; Narodytska, N.; and Marques{-}Silva, J. 2019.
\newblock On Relating Explanations and Adversarial Examples.
\newblock In \emph{NeurIPS}, 15857--15867.

\bibitem[{James et~al.(2017)James, Witten, Hastie, and Tibshirani}]{james-bk17}
James, G.; Witten, D.; Hastie, T.; and Tibshirani, R. 2017.
\newblock \emph{An introduction to statistical learning}.
\newblock Springer.

\bibitem[{Janzing, Minorics, and Bl{\"{o}}baum(2020)}]{blockbaum-aistats20}
Janzing, D.; Minorics, L.; and Bl{\"{o}}baum, P. 2020.
\newblock Feature relevance quantification in explainable {AI:} {A} causal
  problem.
\newblock In \emph{AISTATS}, 2907--2916.

\bibitem[{Jordan and Dimakis(2020)}]{dimakis-nips20}
Jordan, M.; and Dimakis, A.~G. 2020.
\newblock Exactly Computing the Local {L}ipschitz Constant of {ReLU} Networks.
\newblock In \emph{NeurIPS}.

\bibitem[{Kumar et~al.(2021)Kumar, Scheidegger, Venkatasubramanian, and
  Friedler}]{friedler-nips21}
Kumar, I.; Scheidegger, C.; Venkatasubramanian, S.; and Friedler, S.~A. 2021.
\newblock Shapley Residuals: Quantifying the limits of the {S}hapley value for
  explanations.
\newblock In \emph{NeurIPS}, 26598--26608.

\bibitem[{Kumar et~al.(2020)Kumar, Venkatasubramanian, Scheidegger, and
  Friedler}]{friedler-icml20}
Kumar, I.~E.; Venkatasubramanian, S.; Scheidegger, C.; and Friedler, S.~A.
  2020.
\newblock Problems with {S}hapley-value-based explanations as feature
  importance measures.
\newblock In \emph{ICML}, 5491--5500.

\bibitem[{Letoffe et~al.(2024)Letoffe, Huang, Asher, and
  Marques{-}Silva}]{lhams-corr24}
Letoffe, O.; Huang, X.; Asher, N.; and Marques{-}Silva, J. 2024.
\newblock From {SHAP} Scores to Feature Importance Scores.
\newblock \emph{CoRR}, abs/2405.11766.

\bibitem[{Letoffe, Huang, and Marques{-}Silva(2024)}]{lhms-corr24}
Letoffe, O.; Huang, X.; and Marques{-}Silva, J. 2024.
\newblock On Correcting {SHAP} Scores.
\newblock \emph{CoRR}, abs/2405.00076.

\bibitem[{Lundberg and Lee(2017)}]{lundberg-nips17}
Lundberg, S.~M.; and Lee, S. 2017.
\newblock A Unified Approach to Interpreting Model Predictions.
\newblock In \emph{NeurIPS}, 4765--4774.

\bibitem[{Marques{-}Silva and Ignatiev(2022)}]{msi-aaai22}
Marques{-}Silva, J.; and Ignatiev, A. 2022.
\newblock Delivering Trustworthy {AI} through Formal {XAI}.
\newblock In \emph{AAAI}, 12342--12350.

\bibitem[{Mishra(2023)}]{mishra-bk23}
Mishra, P. 2023.
\newblock \emph{Explainable {AI} Recipes}.
\newblock Apress.
\newblock ISBN 978-1-4842-9029-3.

\bibitem[{Molnar(2023)}]{molnar-bk23}
Molnar, C. 2023.
\newblock \emph{Interpreting Machine Learning Models With {SHAP}}.
\newblock Lulu.com.
\newblock ISBN 979-8857734445.

\bibitem[{O'Searcoid(2006)}]{osearcoid-bk06}
O'Searcoid, M. 2006.
\newblock \emph{Metric spaces}.
\newblock Springer Science \& Business Media.

\bibitem[{Pauli et~al.(2022)Pauli, Koch, Berberich, Kohler, and
  Allg\"{o}wer}]{allgower-csl22}
Pauli, P.; Koch, A.; Berberich, J.; Kohler, P.; and Allg\"{o}wer, F. 2022.
\newblock Training Robust Neural Networks Using Lipschitz Bounds.
\newblock \emph{IEEE Control Systems Letters}, 6: 121--126.

\bibitem[{Roth(1988)}]{roth-bk88}
Roth, A.~E. 1988.
\newblock \emph{The {S}hapley value: essays in honor of {L}loyd {S}.
  {S}hapley}.
\newblock Cambridge University Press.

\bibitem[{Shapley(1953)}]{shapley-ctg53}
Shapley, L.~S. 1953.
\newblock A value for $n$-person games.
\newblock \emph{Contributions to the Theory of Games}, 2(28): 307--317.

\bibitem[{Strumbelj and Kononenko(2010)}]{kononenko-jmlr10}
Strumbelj, E.; and Kononenko, I. 2010.
\newblock An Efficient Explanation of Individual Classifications using Game
  Theory.
\newblock \emph{J. Mach. Learn. Res.}, 11: 1--18.

\bibitem[{Strumbelj and Kononenko(2014)}]{kononenko-kis14}
Strumbelj, E.; and Kononenko, I. 2014.
\newblock Explaining prediction models and individual predictions with feature
  contributions.
\newblock \emph{Knowl. Inf. Syst.}, 41(3): 647--665.

\bibitem[{Sundararajan and Najmi(2020)}]{najmi-icml20}
Sundararajan, M.; and Najmi, A. 2020.
\newblock The Many {S}hapley Values for Model Explanation.
\newblock In \emph{ICML}, 9269--9278.

\bibitem[{Szegedy et~al.(2014)Szegedy, Zaremba, Sutskever, Bruna, Erhan,
  Goodfellow, and Fergus}]{szegedy-iclr14}
Szegedy, C.; Zaremba, W.; Sutskever, I.; Bruna, J.; Erhan, D.; Goodfellow,
  I.~J.; and Fergus, R. 2014.
\newblock Intriguing properties of neural networks.
\newblock In \emph{ICLR}.

\bibitem[{{US Government}(2023)}]{wheo23}
{US Government}. 2023.
\newblock Executive Order on the Safe, Secure, and Trustworthy Development and
  Use of Artificial Intelligence.
\newblock \url{https://tinyurl.com/yb7wcmen}.

\bibitem[{{Van den Broeck} et~al.(2021){Van den Broeck}, Lykov, Schleich, and
  Suciu}]{vandenbroeck-aaai21}
{Van den Broeck}, G.; Lykov, A.; Schleich, M.; and Suciu, D. 2021.
\newblock On the Tractability of {SHAP} Explanations.
\newblock In \emph{AAAI}, 6505--6513.

\bibitem[{{Van den Broeck} et~al.(2022){Van den Broeck}, Lykov, Schleich, and
  Suciu}]{vandenbroeck-jair22}
{Van den Broeck}, G.; Lykov, A.; Schleich, M.; and Suciu, D. 2022.
\newblock On the Tractability of {SHAP} Explanations.
\newblock \emph{J. Artif. Intell. Res.}, 74: 851--886.

\bibitem[{Virmaux and Scaman(2018)}]{scaman-nips18}
Virmaux, A.; and Scaman, K. 2018.
\newblock Lipschitz regularity of deep neural networks: analysis and efficient
  estimation.
\newblock In \emph{NeurIPS}, 3839--3848.

\bibitem[{W{\"{a}}ldchen et~al.(2021)W{\"{a}}ldchen, MacDonald, Hauch, and
  Kutyniok}]{kutyniok-jair21}
W{\"{a}}ldchen, S.; MacDonald, J.; Hauch, S.; and Kutyniok, G. 2021.
\newblock The Computational Complexity of Understanding Binary Classifier
  Decisions.
\newblock \emph{J. Artif. Intell. Res.}, 70: 351--387.

\bibitem[{Weng et~al.(2018)Weng, Zhang, Chen, Yi, Su, Gao, Hsieh, and
  Daniel}]{daniel-iclr18}
Weng, T.; Zhang, H.; Chen, P.; Yi, J.; Su, D.; Gao, Y.; Hsieh, C.; and Daniel,
  L. 2018.
\newblock Evaluating the Robustness of Neural Networks: An Extreme Value Theory
  Approach.
\newblock In \emph{ICLR}.

\bibitem[{Wu, Wu, and Barrett(2023)}]{barrett-nips23}
Wu, M.; Wu, H.; and Barrett, C.~W. 2023.
\newblock {VeriX}: Towards Verified Explainability of Deep Neural Networks.
\newblock In \emph{NeurIPS}.

\bibitem[{Young(1985)}]{young-ijgt85}
Young, H.~P. 1985.
\newblock Monotonic solutions of cooperative games.
\newblock \emph{International Journal of Game Theory}, 14: 65--72.

\bibitem[{Yu et~al.(2023)Yu, Farr, Ignatiev, and Stuckey}]{ignatiev-corr23b}
Yu, J.; Farr, G.; Ignatiev, A.; and Stuckey, P.~J. 2023.
\newblock Anytime Approximate Formal Feature Attribution.
\newblock \emph{CoRR}, abs/2312.06973.

\bibitem[{Yu, Ignatiev, and Stuckey(2023)}]{ignatiev-corr23a}
Yu, J.; Ignatiev, A.; and Stuckey, P.~J. 2023.
\newblock On Formal Feature Attribution and Its Approximation.
\newblock \emph{CoRR}, abs/2307.03380.

\bibitem[{Zhang et~al.(2022)Zhang, Jiang, He, and Wang}]{wang-nips22}
Zhang, B.; Jiang, D.; He, D.; and Wang, L. 2022.
\newblock Rethinking {L}ipschitz Neural Networks and Certified Robustness: {A}
  Boolean Function Perspective.
\newblock In \emph{NeurIPS}.

\end{thebibliography}
}{

}

\clearpage

\appendix

\onecolumn

\section{Supplemental Materials} \label{sec:app}

\subsection{Proofs for Propositions in Classification -- Boolean Domains} \label{ssec:proofs}
%


\PropIRR*

\begin{proof}
Let $\fml{M}$ be a classifier defined on the feature set $\fml{F}$ and characterized by the function defined as follows:
\begin{equation}
\kappa(\mbf{x}_{1..m}, x_n) :=
\left \{
\begin{array}{lcl}
\kappa_1(\mbf{x}_{1..m}) \lor f(\mbf{x}_{1..m}) & \quad & \tn{if~$x_n=0$} \\[3pt]
\kappa_1(\mbf{x}_{1..m}) & \quad & \tn{if~$x_n=1$}
\end{array}
\right.
\end{equation}
The non-constant sub-functions $\kappa_1$ and $f$ are defined on the feature set $\fml{F}\setminus\{n\}$,
and satisfy the following conditions:
\begin{enumerate}[itemsep=-.5pt]
\item
$\kappa_1 \neq \kappa_1 \lor f$ and $\kappa_1 \land f = 0$.
\item
Both $\kappa_1$ and $\kappa_1 \lor f$ predict a specific point $\mbf{v}_{1..m}$ to 0.
\item
The set of CXps for both $\kappa_1$ and $\kappa_1 \lor f$ with respect to the point $\mbf{v}_{1..m}$ are identical.
\end{enumerate}
Choose this specific $m$-dimensional point $\mbf{v}_{1..m}$ and extend it with $v_n=1$.
This means $\kappa_0(\mbf{v}) = \kappa_1(\mbf{v}) = 0$, and therefore $\kappa(\mbf{v}) = 0$.
For any $\fml{S} \subseteq \fml{F}\setminus\{n\}$, we have
\begin{equation}
\begin{aligned}
\Delta_{n}(\fml{S};\fml{E},\cfn{e})
&= \frac{1}{2} \cdot ( \exv[\kappa_1|\mbf{x}_{\fml{S}}=\mbf{v}_{\fml{S}}] - \exv[(\kappa_1 \lor f)|\mbf{x}_{\fml{S}}=\mbf{v}_{\fml{S}}] ) \\
&= \frac{1}{2} \cdot ( \exv[\kappa_1|\mbf{x}_{\fml{S}}=\mbf{v}_{\fml{S}}] - \exv[\kappa_1|\mbf{x}_{\fml{S}}=\mbf{v}_{\fml{S}}] - \exv[f|\mbf{x}_{\fml{S}}=\mbf{v}_{\fml{S}}]) \\
&= \frac{1}{2} \cdot ( - \exv[f|\mbf{x}_{\fml{S}}=\mbf{v}_{\fml{S}}]  ),
\end{aligned}
\end{equation}
we can infer that
$ - \exv[f|\mbf{x}_{\fml{S}}=\mbf{v}_{\fml{S}}] < 0$
for some $\fml{S}$, which implies $\svn{E}(n) < 0$.

To prove that feature $n$ is irrelevant, 
we assume the contrary that feature $n$ is relevant, and $\fml{X}$, where $n \in \fml{X}$, is an AXp of the point $\mbf{v}$.
Based on the definition of AXp, 
we only include points $\mbf{x}$ for which $\kappa_1(\mbf{x})=0$ holds.
As $\kappa_1$ and $\kappa_1 \lor f$ share the same set of CXps, they have the same set of AXps.
This means 
$\fml{X} \setminus \{n\}$ will not include any points $\mbf{x}$ such that 
either $\kappa_1(\mbf{x})\neq 0$ or $(\kappa_1 \lor f)(\mbf{x})\neq 0$ holds.
This means $\fml{X} \setminus \{n\}$ remains an AXp of the point $\mbf{v}$, leading to a contradiction.
Thus, feature $n$ is irrelevant.
\end{proof}


\PropREL*

\begin{proof}
Let $\fml{M}$ be a classifier defined on the feature set $\fml{F}$ and characterized by the function defined as follows:
\begin{equation}
\kappa(\mbf{x}_{1..m},\mbf{x}_{m+1..2m},x_n) :=
\left \{
\begin{array}{lcl}
\kappa_0(\mbf{x}_{1..m}) & \quad & \tn{if~$x_n=0$} \\[3pt]
\kappa_1(\mbf{x}_{m+1..2m}) & \quad & \tn{if~$x_n=1$}
\end{array}
\right.
\end{equation}
The non-constant sub-functions $\kappa_0$ and $\kappa_1$ are defined on the feature sets $\fml{F}_0=\{1,\dots,m\}$ and $\fml{F}_1=\{m+1,\dots,2m\}$, respectively.
It is important to note that $\kappa_0$ is independent of $\kappa_1$ as $\fml{F}_0$ and $\fml{F}_1$ are disjoint.
Moreover, $\kappa_0$ and $\kappa_1$ are identical up to isomorphism.
(For simplicity, we assume that feature $i$ corresponds to feature $m+i$ for all $i \in \{1,\dots,m\}$.)

Choose a $n$-dimensional point $\mbf{v}$ such that: 1) $v_n=1$, 2) $v_{i} = v_{m+i}$ for any $1 \le i \le m$, 
and 3) $\kappa_0(\mbf{v}) = \kappa_1(\mbf{v}) = 1$.
This means $\kappa(\mbf{v}) = 1$.
For any $\fml{S}\subset\fml{F}\setminus\{n\}$ such that $\fml{S}\neq \emptyset$,
let $\{\fml{S}_0, \fml{S}_1\}$ be a partition of $\fml{S}$
such that $\fml{S}_0\subseteq\fml{F}_0$ and $\fml{S}_1\subseteq\fml{F}_1$, then
\begin{equation}
\begin{aligned}
\Delta_{n}(\fml{S};\fml{E},\cfn{e})
&= \frac{1}{2} \cdot ( \exv[\kappa_1|\mbf{x}_{\fml{S}_1}=\mbf{v}_{\fml{S}_1}] - \exv[\kappa_0|\mbf{x}_{\fml{S}_0}=\mbf{v}_{\fml{S}_0}] ).
\end{aligned}
\end{equation}
For any $\{\fml{S}_0, \fml{S}_1\}$, we can construct a unique new partition $\{\fml{S}'_0, \fml{S}'_1\}$
by replacing any $i \in \fml{S}_0$ with $m+i$ and any $m+i \in \fml{S}_1$ with $i$.
Let $\fml{S}' = \fml{S}'_0 \cup \fml{S}'_1$, then we have
\begin{equation}
\begin{aligned}
\Delta_{n}(\fml{S}';\fml{E},\cfn{e})
&= \frac{1}{2} \cdot ( \exv[\kappa_1|\mbf{x}_{\fml{S}'_0}=\mbf{v}_{\fml{S}'_0}] - \exv[\kappa_0|\mbf{x}_{\fml{S}'_1}=\mbf{v}_{\fml{S}'_1}] ).
\end{aligned}
\end{equation}
Besides, we have
$\exv[\kappa_1|\mbf{x}_{\fml{S}_1}=\mbf{v}_{\fml{S}_1}] = \exv[\kappa_0|\mbf{x}_{\fml{S}'_1}=\mbf{v}_{\fml{S}'_1}]$
and $\exv[\kappa_0|\mbf{x}_{\fml{S}_0}=\mbf{v}_{\fml{S}_0}] = \exv[\kappa_1|\mbf{x}_{\fml{S}'_0}=\mbf{v}_{\fml{S}'_0}]$,
which means
\begin{equation}
\Delta_{n}(\fml{S};\fml{E},\cfn{e}) = - \Delta_{n}(\fml{S}';\fml{E},\cfn{e}),
\end{equation}
note that 
$\varsigma(\fml{S}) = \varsigma(\fml{S}')$.
Hence, for any $\fml{S}\subset\fml{F}\setminus\{n\}$ such that $\fml{S}\neq \emptyset$,
there is a unique $\fml{S}'$ that can cancel its effect.
Besides, if $\fml{S} = \emptyset$ or $\fml{S} = \fml{F}\setminus\{n\}$, then we have $\Delta_{n}(\fml{S};\fml{E},\cfn{e}) = 0$.
We can derive that $\svn{E}(n) = 0$.
However, $n$ is a relevant feature.
To prove this, it is evident that $\fml{F}\setminus\fml{F}_0$ represents a weak AXp.
Moreover, $\fml{F}\setminus(\fml{F}_0\cup\{n\})$ is not a weak AXp
because allowing $x_n$ to take the value 0 will include points $\mbf{x}$ such that $\kappa_0(\mbf{x})\neq 1$.
Hence, there are AXps containing feature $n$.
\end{proof}


\PropDisorder*

\begin{proof}
Let $\fml{M}$ be a classifier defined on the feature set $\fml{F}$ and characterized by the function defined as follows:
\begin{equation}
\kappa(\mbf{x}_{1..m},\mbf{x}_{m+1..2m},x_{n-1},x_{n}) :=
\left \{
\begin{array}{lcl}
\kappa_{00}(\mbf{x}_{1..m}) & \quad & \tn{if~$x_n=0 \land x_{n-1}=0$} \\[3pt]
\kappa_{01}(\mbf{x}_{m+1..2m}) & \quad & \tn{if~$x_n=0 \land x_{n-1}=1$} \\[3pt]
\kappa_0(\mbf{x}_{1..2m,x_{n-1}}) \lor f(\mbf{x}_{1..2m}) & \quad & \tn{if~$x_n=1$}
\end{array}
\right.
\end{equation}
The non-constant sub-functions $\kappa_{00}$, $\kappa_{01}$ and $f$ are defined on the feature sets
$\fml{F}_0=\{1,\dots,m\}$, $\fml{F}_1=\{m+1,\dots,2m\}$, and $\fml{F}\setminus\{n-1,n\}$, respectively.
It is worth noting that $\kappa_{00}$ is independent of $\kappa_{01}$ as $\fml{F}_0$ and $\fml{F}_1$ are disjoint.
Also note that $\kappa_1 = \kappa_0 \lor f$.
Moreover, $\kappa_{00}$, $\kappa_{01}$ and $f$ satisfy the following conditions:
\begin{enumerate}[itemsep=-.5pt]
\item
$\kappa_{00}$ and $\kappa_{01}$ are identical up to isomorphism.
(For simplicity, we assume that feature $i$ corresponds to feature $m+i$ for all $i \in \{1,\dots,m\}$.)
\item
$\kappa_0 \neq \kappa_0 \lor f$, $\kappa_{00} \land f = 0$ and $\kappa_{01} \land f = 0$.
\item
Both $\kappa_0$ and $\kappa_0 \lor f$ predict a specific point $\mbf{v}_{1..n-1}$ to 1,
where $v_{n-1}=1$, and $v_{i} = v_{m+i}$ for any $1 \le i \le m$.
\item
The set of CXps for $\kappa_0$ and $\kappa_0 \lor f$ with respect to the point $\mbf{v}_{1..n-1}$ are identical.
\end{enumerate}
Choose this specific $n-1$-dimensional point $\mbf{v}_{1..n-1}$ and extend it with $v_n=1$,
then $\kappa_0(\mbf{v}) = \kappa_1(\mbf{v}) = 1$ and $\kappa(\mbf{v}) = 1$.
For any $\fml{S} \subseteq \fml{F}\setminus\{n\}$, we have
\begin{equation}
\begin{aligned}
\Delta_{n}(\fml{S};\fml{E},\cfn{e})
&= \frac{1}{2} \cdot ( \exv[(\kappa_0 \lor f)|\mbf{x}_{\fml{S}}=\mbf{v}_{\fml{S}}] -\exv[\kappa_0|\mbf{x}_{\fml{S}}=\mbf{v}_{\fml{S}}] ) \\
&= \frac{1}{2} \cdot ( \exv[f|\mbf{x}_{\fml{S}}=\mbf{v}_{\fml{S}}] ),
\end{aligned}
\end{equation}
which implies $\svn{E}(n) > 0$.
As $\kappa_0$ and $\kappa_0 \lor f$ share the same set of CXps, they have the same set of AXps.
By applying similar reasoning as presented in the proof of~\cref{prop:irr},
we can conclude that feature $n$ is irrelevant.
For any $\fml{S}\subset\fml{F}\setminus\{n-1,n\}$ such that $\fml{S} \neq \emptyset$, we have
\begin{equation}
\begin{aligned}
\Delta_{n-1}(\fml{S};\fml{E},\cfn{e})
&= \frac{1}{2} \cdot ( \frac{1}{2} \cdot ( \exv[\kappa_{01}|\mbf{x}_{\fml{S}}=\mbf{v}_{\fml{S}}] - \exv[\kappa_{00}|\mbf{x}_{\fml{S}}=\mbf{v}_{\fml{S}}] )
+ \frac{1}{2} \cdot ( \exv[\kappa_{11}|\mbf{x}_{\fml{S}}=\mbf{v}_{\fml{S}}] - \exv[\kappa_{10}|\mbf{x}_{\fml{S}}=\mbf{v}_{\fml{S}}] ) ) \\
&= \frac{1}{2} \cdot ( \exv[\kappa_{01}|\mbf{x}_{\fml{S}}=\mbf{v}_{\fml{S}}] - \exv[\kappa_{00}|\mbf{x}_{\fml{S}}=\mbf{v}_{\fml{S}}] ),
\end{aligned}
\end{equation}
besides, we have
\begin{equation}
\begin{aligned}
\Delta_{n-1}(\fml{S}\cup\{n\};\fml{E},\cfn{e})
&= \frac{1}{2} \cdot ( \exv[\kappa_{11}|\mbf{x}_{\fml{S}}=\mbf{v}_{\fml{S}}] - \exv[\kappa_{10}|\mbf{x}_{\fml{S}}=\mbf{v}_{\fml{S}}] ) \\
&= \frac{1}{2} \cdot ( \exv[\kappa_{01}|\mbf{x}_{\fml{S}}=\mbf{v}_{\fml{S}}] - \exv[\kappa_{00}|\mbf{x}_{\fml{S}}=\mbf{v}_{\fml{S}}] ),
\end{aligned}
\end{equation}
also note that $\Delta_{n-1}(\fml{S};\fml{E},\cfn{e}) = 0$ when $\fml{S} = \emptyset$ or $\fml{S} = \fml{F}\setminus\{n-1\}$.
By applying the same reasoning as presented in the proof of~\cref{prop:rel},
for any $\fml{S}$, there is a unique $\fml{S}'$ such that $|\fml{S}| = |\fml{S}'|$ to cancel the effect of $\fml{S}$.
Thus, we can conclude that $\svn{E}(n-1)=0$ but feature $n-1$ is relevant.
\end{proof}


\PropHighest*

\begin{proof}
Let $\fml{M}$ be a classifier defined on the feature set $\fml{F}$ and characterized by the function defined as follows:
\begin{equation}
\kappa(\mbf{x}_{1..m},x_n) :=
\left \{
\begin{array}{lcl}
0 & \quad & \tn{if~$x_n=0$} \\
\kappa_1(\mbf{x}_{1..m}) & \quad & \tn{if~$x_n=1$}
\end{array}
\right.
\end{equation}
Its sub-function $\kappa_1$ is a non-constant function defined on the feature set $\fml{F}\setminus\{n\}$,
and satisfies the following conditions:
\begin{enumerate}[itemsep=-.5pt]
\item
$\kappa_1$ predicts a specific point $\mbf{v}_{1..m}$ to 0.
\item
For any point $\mbf{x}_{1..m}$ such that $||\mbf{x}_{1..m} - \mbf{v}_{1..m}||_{0} = 1$,
we have $\kappa_1(\mbf{x}_{1..m}) = 1$.
\item
$\kappa_1$ predicts all the other points to 0.
\end{enumerate}
For example, $\kappa_1$ can be the function $\sum_{i=1}^{m}\neg{x_i}=1$,
which predicts the point $\mbf{1}_{1..m}$ to 0 and all points around this point with a Hamming distance of 1 to 1.

Select this specific $m$-dimensional point $\mbf{v}_{1..m}$ such that $\kappa_1(\mbf{v}_{1..m}) = 0$.
Extend $\mbf{v}_{1..m}$ with $v_n=1$, we have $\kappa(\mbf{v}) = 0$.
To prove that the feature $n$ is irrelevant, 
we assume the contrary that the feature $n$ is relevant, and $\fml{X}$, where $n \in \fml{X}$, is an AXp of the point $\mbf{v}$.
Based on the definition of AXp, 
we only include points $\mbf{x}$ for which $\kappa_1(\mbf{x})=0$ holds.
However, as $\kappa_0=0$, 
$\fml{X} \setminus \{n\}$ will not include any points $\mbf{x}$ such that either $\kappa_0(\mbf{x})\neq 0$ or $\kappa_1(\mbf{x})\neq 0$ holds.
This means $\fml{X} \setminus \{n\}$ remains an AXp of the point $\mbf{v}$, leading to a contradiction.
Thus, feature $n$ is irrelevant.
In addition, for $\kappa_1$ and any $\fml{S}\subseteq\fml{F}\setminus\{n\}$, we have
\begin{equation}
\exv[\kappa_1|\mbf{x}_{\fml{S}}=\mbf{v}_{\fml{S}}] = \frac{m-|\fml{S}|}{2^{m-|\fml{S}|}}.
\end{equation}
For feature $n$ and an arbitrary $\fml{S}\subseteq\fml{F}$, we have
\begin{equation}
\begin{aligned}
\Delta_{n}(\fml{S};\fml{E},\cfn{e})
= \frac{1}{2} \cdot \frac{m-|\fml{S}|}{2^{m-|\fml{S}|}},
\end{aligned}
\end{equation}
this means $\svn{E}(n) > 0$.
Besides, the unique minimal value of $\Delta_{n}(\fml{S};\fml{E},\cfn{e})$ is 0
when $\fml{S} = \fml{F}\setminus\{n\}$.

We now focus on a feature $j \neq n$.
Consider an arbitrary $\fml{S}\subseteq\fml{F}\setminus\{j,n\}$, we have
\begin{equation}
\begin{aligned}
\Delta_{j}(\fml{S}\cup\{n\};\fml{E},\cfn{e})
&= \frac{m-|\fml{S}|-1}{2^{m-|\fml{S}|-1}} - \frac{m-|\fml{S}|}{2^{m-|\fml{S}|}} \\
&= \frac{m-|\fml{S}|-2}{2^{m-|\fml{S}|}}.
\end{aligned}
\end{equation}
In this case, 
$\Delta_{j}(\fml{S}\cup\{n\};\fml{E},\cfn{e}) = - \frac{1}{2}$ if $|\fml{S}| = m-1$, which is its unique minimal value.
$\Delta_{j}(\fml{S}\cup\{n\};\fml{E},\cfn{e}) = 0$ if $|\fml{S}| = m-2$,
and $\Delta_{j}(\fml{S}\cup\{n\};\fml{E},\cfn{e}) > 0$ if $|\fml{S}| < m-2$.
Besides, we have
\begin{equation}
\begin{aligned}
\Delta_{j}(\fml{S};\fml{E},\cfn{e})
= \frac{1}{2} \cdot \frac{m-|\fml{S}|-2}{2^{m-|\fml{S}|}}.
\end{aligned}
\end{equation}
In this case, $\Delta_{j}(\fml{S};\fml{E},\cfn{e}) = - \frac{1}{4}$ if $|\fml{S}| = m-1$, which is its unique minimal value.
$\Delta_{j}(\fml{S};\fml{E},\cfn{e}) = 0$ if $|\fml{S}| = m-2$,
and $\Delta_{j}(\fml{S};\fml{E},\cfn{e}) > 0$ if $|\fml{S}| < m-2$.

Next, we prove $|\svn{E}(n)| > |\svn{E}(j)|$ by showing
$\svn{E}(n) + \svn{E}(j) > 0$ and $\svn{E}(n) - \svn{E}(j) > 0$.
Note that $\svn{E}(n) > 0$.
Additionally, $\Delta_{j}(\fml{S}\cup\{n\};\fml{E},\cfn{e}) < 0$
and $\Delta_{j}(\fml{S};\fml{E},\cfn{e}) < 0$
only when $|\fml{S}| = m-1$.
Compute the SHAP score for feature $n$:
\begin{equation}
\begin{aligned}
\svn{E}(n)
&= \sum_{\fml{S}\subseteq \fml{F}\setminus\{n\}}\frac{|\fml{S}|!(m-|\fml{S}|)!}{(m+1)!} \cdot 
\Delta_{n}(\fml{S};\fml{E},\cfn{e}) \\
&= \sum_{\fml{S}\subseteq \fml{F}\setminus\{n\}}\frac{|\fml{S}|!(m-|\fml{S}|)!}{(m+1)!} \cdot \frac{1}{2} \cdot \frac{m-|\fml{S}|}{2^{m-|\fml{S}|}} \\
&= \frac{1}{2} \cdot \frac{1}{m+1} \cdot \sum_{\fml{S}\subseteq \fml{F}\setminus\{n\}}\frac{|\fml{S}|!(m-|\fml{S}|)!}{m!} 
\cdot \frac{m-|\fml{S}|}{2^{m-|\fml{S}|}} \\
&= \frac{1}{2} \cdot \frac{1}{m+1} \cdot \sum_{0 \le |\fml{S}| \le m}\frac{|\fml{S}|!(m-|\fml{S}|)!}{m!} \cdot \frac{m!}{|\fml{S}|!(m-|\fml{S}|)!}
\cdot \frac{m-|\fml{S}|}{2^{m-|\fml{S}|}} \\
&= \frac{1}{2} \cdot \frac{1}{m+1} \cdot \sum^{m}_{k=1}\frac{k}{2^{k}} \\
&= \frac{1}{2} \cdot \frac{1}{m+1} \cdot \frac{2^{m+1} - m - 2}{2^{m}}  \\
&= \frac{1}{m+1} \cdot \frac{2^{m+1} - m - 2}{2^{m+1}}.
\end{aligned}
\end{equation}
Now we focus on a feature $j \neq n$.
Consider the subset $\fml{S}=\fml{F}\setminus\{j,n\}$ where $|\fml{S}| = m-1$, we have
\begin{equation}
\begin{aligned}
&\frac{|\fml{S}\cup\{n\}|!(m-|\fml{S}\cup\{n\}|)!}{(m+1)!} \cdot \frac{m-|\fml{S}|-2}{2^{m-|\fml{S}|}} \\
&= - \frac{1}{2} \cdot \frac{1}{m+1},
\end{aligned}
\end{equation}
moreover, we have
\begin{equation}
\begin{aligned}
&\frac{|\fml{S}|!(m-|\fml{S}|)!}{(m+1)!} \cdot \frac{1}{2} \cdot \frac{m-|\fml{S}|-2}{2^{m-|\fml{S}|}} \\
&= - \frac{1}{4} \cdot \frac{1}{m(m+1)}.
\end{aligned}
\end{equation}
The sum of these three values is
\begin{equation}
\begin{aligned}
&\frac{1}{m+1} \cdot \frac{2^{m+1} - m - 2}{2^{m+1}} - \frac{1}{2} \cdot \frac{1}{m+1} - \frac{1}{4} \cdot \frac{1}{m(m+1)} \\
&= \frac{1}{m+1} \cdot \left( \frac{(2^{m+1} - m - 2)m}{m2^{m+1}} - \frac{m2^m}{m2^{m+1}} - \frac{2^{m-1}}{m2^{m+1}} \right) \\
&= \frac{1}{m(m+1)2^{m+1}} \cdot \left( (m - \frac{1}{2})2^m - m^2 - 2m \right),
\end{aligned}
\end{equation}
since $m \ge 3$, the sum of these three values is always greater than 0.
Thus, we can conclude that $\svn{E}(n) + \svn{E}(j) > 0$.

To show $\svn{E}(n) - \svn{E}(j) > 0$, we focus on all $\fml{S}\subseteq\fml{F}\setminus\{n\}$ where $|\fml{S}| < m-2$.
This is because, as previously stated,
$\Delta_{j}(\fml{S}\cup\{n\};\fml{E},\cfn{e}) \le 0$
and $\Delta_{j}(\fml{S};\fml{E},\cfn{e}) \le 0$ if $|\fml{S}| \ge m-2$.

Moreover, for all $\fml{S}\subseteq\fml{F}\setminus\{n\}$ where $|\fml{S}|=k$ and $0 < k \le m-3$, we compute the following three quantities:
\begin{equation}
\begin{aligned}
Q_1 &:= \sum_{\fml{S}\subseteq\fml{F}\setminus\{n\}, |\fml{S}|=k} \Delta_{n}(\fml{S};\fml{E},\cfn{e}), \\
Q_2 &:= \sum_{\fml{S}\subseteq\fml{F}\setminus\{j,n\}, |\fml{S}|=k-1} \Delta_{j}(\fml{S}\cup\{n\};\fml{E},\cfn{e}), \\
Q_3 &:= \sum_{\fml{S}\subseteq\fml{F}\setminus\{j,n\}, |\fml{S}|=k} \Delta_{j}(\fml{S};\fml{E},\cfn{e}),
\end{aligned}
\end{equation}
and show that $Q_1 - Q_2 - Q_3 > 0$.
Note that $Q_1$, $Q_2$ and $Q_3$ share the same coefficient $\frac{k!(n-k-1)!}{n!}$.
For feature $n$, we pick all possible $\fml{S}\subseteq\fml{F}\setminus\{n\}$ where $|\fml{S}| = k$, which implies $|\fml{S} \cup \{n\}| = k+1$, then
\begin{equation}
Q_1 = \binom{m}{|\fml{S}|} \cdot \frac{1}{2} \cdot \frac{m-|\fml{S}|}{2^{m-|\fml{S}|}} 
= \binom{m}{k} \cdot \frac{1}{2} \cdot \frac{m-k}{2^{m-k}}.
\end{equation}
For a feature $j \neq n$.
We pick all possible $\fml{S}\subseteq\fml{F}\setminus\{j,n\}$ where $|\fml{S}| = k-1$, which implies $|\fml{S} \cup \{j, n\}| = k+1$, then
\begin{equation}
Q_2 = \binom{m-1}{|\fml{S}|} \cdot \frac{m-|\fml{S}|-2}{2^{m-|\fml{S}|}} 
= \binom{m-1}{k-1} \cdot \frac{1}{2} \cdot \frac{m-k-1}{2^{m-k}}.
\end{equation}
We pick all possible $\fml{S}\subseteq\fml{F}\setminus\{j,n\}$ where $|\fml{S}| = k$, which implies $|\fml{S} \cup \{j\}| = k+1$, then
\begin{equation}
Q_3 = \binom{m-1}{|\fml{S}|} \cdot \frac{1}{2} \cdot \frac{m-|\fml{S}|-2}{2^{m-|\fml{S}|}} 
= \binom{m-1}{k} \cdot \frac{1}{2} \cdot \frac{m-k-2}{2^{m-k}}.
\end{equation}
Then we compute $Q_1 - Q_2 - Q_3$:
\begin{equation}
\begin{aligned}
&\binom{m}{k} \cdot \frac{1}{2} \cdot \frac{m-k}{2^{m-k}} - \binom{m-1}{k-1} \cdot \frac{1}{2} \cdot \frac{m-k-1}{2^{m-k}}
- \binom{m-1}{k} \cdot \frac{1}{2} \cdot \frac{m-k-2}{2^{m-k}} \\
&= \frac{1}{2} \cdot \frac{1}{2^{m-k}} 
\left[ \binom{m}{k} (m-k) - \binom{m-1}{k-1} (m - k - 1) - \binom{m-1}{k} (m - k - 2) \right] \\
&= \frac{1}{2} \cdot \frac{1}{2^{m-k}} \left[ \binom{m-1}{k-1} + 2 \binom{m-1}{k} \right],
\end{aligned}
\end{equation}
this means that $\svn{E}(n) - \svn{E}(j) > 0$.
Hence, we can conclude that $|\svn{E}(n)| > |\svn{E}(j)|$.
\end{proof}


\PropSign*

\begin{proof}
Let $\fml{M}$ be a classifier defined on the feature set $\fml{F}$ and characterized by the function defined as follows:
\begin{equation}
\kappa(\mbf{x}_{1..m},x_{n-1},x_{n}) :=
\left \{
\begin{array}{lcl}
\kappa'(\mbf{x}_{1..m}) & \quad & \tn{if~$x_n=0 \land x_{n-1}=0$}\\[3pt]
\kappa'(\mbf{x}_{1..m}) \lor f(\mbf{x}_{1..m}) & \quad & \tn{if~$x_n=0 \land x_{n-1}=1$}\\[3pt]
\kappa'(\mbf{x}_{1..m}) \lor g(\mbf{x}_{1..m}) & \quad & \tn{if~$x_n=1 \land x_{n-1}=0$}\\[3pt]
\kappa'(\mbf{x}_{1..m}) \lor f(\mbf{x}_{1..m}) \lor g(\mbf{x}_{1..m}) & \quad & \tn{if~$x_n=1 \land x_{n-1}=1$}
\end{array}
\right.
\end{equation}
The non-constant sub-functions $\kappa'$, $f$ and $g$ are defined on the feature set $\fml{F}\setminus\{n\}$.
Moreover, $\kappa'$, $f$ and $g$ satisfy the following conditions:
\begin{enumerate}[itemsep=-.5pt]
\item
$\kappa_0 \neq \kappa_1$, 
$\kappa' \land f = 0$, $\kappa' \land g = 0$ and $f \land g = 0$.
\item
$f$ predicts a specific point $\mbf{v}_{1..m}$ to 1.
\item
$\kappa'$ and $g$ predict this specific point $\mbf{v}_{1..m}$ to 0.
\item
The set of CXps for $\kappa_0$ and $\kappa_1$ with respect to the point $\mbf{v}_{1..n-1}=(v_1, \dots, v_m,1)$ are identical.
\end{enumerate}
Choose the specific $m$-dimensional point $\mbf{v}_{1..m}$ that $f$ predicts to 1, and extend it with $v_{n-1}=v_n=1$,
we have $\kappa_{00}(\mbf{v}) = \kappa_{10}(\mbf{v}) = 0$ and $\kappa_{01}(\mbf{v}) = \kappa_{11}(\mbf{v}) = 1$,
which means $\kappa(\mbf{v}) = 1$.
As $\kappa_0$ and $\kappa_1$ share the same set of CXps, they have the same set of AXps. 
By applying similar reasoning as presented in the proof of~\cref{prop:irr}, we can conclude that feature $n$ is irrelevant.
To prove that feature $n-1$ is relevant, we assume the contrary that $n-1$ is irrelevant.
In this case, we can flip the value $v_{n-1}$ from 1 to 0 and pick the point $\mbf{v}'=(v_1, \dots, v_m, 0, 1)$ which predicted to 0
by the function $\kappa_{10}$. This means $\kappa(\mbf{v}')=0$ , leading to a contradiction. Thus, feature $n-1$ is relevant.

Next, we analyse the SHAP scores of these two features.
For feature $n$, consider an arbitrary $\fml{S}\subseteq\fml{F}\setminus\{n-1,n\}$, we have
\begin{equation}
\begin{aligned}
\Delta_{n}(\fml{S};\fml{E},\cfn{e})
&= \frac{1}{2} \cdot ( \frac{1}{2} \cdot (\exv[\kappa_{10}|\mbf{x}_{\fml{S}}=\mbf{v}_{\fml{S}}] - \exv[\kappa_{00}|\mbf{x}_{\fml{S}}=\mbf{v}_{\fml{S}}])
+ \frac{1}{2} \cdot ( \exv[\kappa_{11}|\mbf{x}_{\fml{S}}=\mbf{v}_{\fml{S}}] - \exv[\kappa_{01}|\mbf{x}_{\fml{S}}=\mbf{v}_{\fml{S}}] ) ) \\
&= \frac{1}{2} \cdot \exv[g|\mbf{x}_{\fml{S}}=\mbf{v}_{\fml{S}}],
\end{aligned}
\end{equation}
moreover, we have
\begin{equation}
\begin{aligned}
\Delta_{n}(\fml{S}\cup\{n-1\};\fml{E},\cfn{e})
&= \frac{1}{2} \cdot ( \exv[\kappa_{11}|\mbf{x}_{\fml{S}}=\mbf{v}_{\fml{S}}] - \exv[\kappa_{01}|\mbf{x}_{\fml{S}}=\mbf{v}_{\fml{S}}] ) \\
&= \frac{1}{2} \cdot \exv[g|\mbf{x}_{\fml{S}}=\mbf{v}_{\fml{S}}].
\end{aligned}
\end{equation}
For feature $n-1$, consider an arbitrary $\fml{S}\subseteq\fml{F}\setminus\{n-1,n\}$, we have
\begin{equation}
\begin{aligned}
\Delta_{n-1}(\fml{S};\fml{E},\cfn{e})
&= \frac{1}{2} \cdot ( \frac{1}{2} \cdot ( \exv[\kappa_{01}|\mbf{x}_{\fml{S}}=\mbf{v}_{\fml{S}}] - \exv[\kappa_{00}|\mbf{x}_{\fml{S}}=\mbf{v}_{\fml{S}}] )
+ \frac{1}{2} \cdot ( \exv[\kappa_{11}|\mbf{x}_{\fml{S}}=\mbf{v}_{\fml{S}}] - \exv[\kappa_{10}|\mbf{x}_{\fml{S}}=\mbf{v}_{\fml{S}}] ) ) \\
&= \frac{1}{2} \cdot \exv[f|\mbf{x}_{\fml{S}}=\mbf{v}_{\fml{S}}],
\end{aligned}
\end{equation}
and we have
\begin{equation}
\begin{aligned}
\Delta_{n-1}(\fml{S}\cup\{n\};\fml{E},\cfn{e})
&= \frac{1}{2} \cdot ( \exv[\kappa_{11}|\mbf{x}_{\fml{S}}=\mbf{v}_{\fml{S}}] - \exv[\kappa_{10}|\mbf{x}_{\fml{S}}=\mbf{v}_{\fml{S}}] ) \\
&= \frac{1}{2} \cdot \exv[f|\mbf{x}_{\fml{S}}=\mbf{v}_{\fml{S}}].
\end{aligned}
\end{equation}
Clearly, by adjusting the sub-functions $f$ and $g$, we are able to change the magnitude of the SHAP scores of both features. 
Importantly, in all cases, their SHAP scores have the same sign.
\end{proof}


\PropEvenly*

\begin{proof}
Let $\kappa'$ be the classification function of $\fml{M}'$, i.e. $\kappa = \neg \kappa'$.
Note that $\cfn{e}(\fml{S};\fml{E}) = \exv[\kappa|\mbf{x}_{\fml{S}}=\mbf{v}_{\fml{S}}]$
and $\cfn{e}(\fml{S};\fml{E}') = \exv[\kappa'|\mbf{x}_{\fml{S}}=\mbf{v}_{\fml{S}}]$.
Evidently, $\exv[\kappa|\mbf{x}_{\fml{S}}=\mbf{v}_{\fml{S}}] = 1 - \exv[\kappa'|\mbf{x}_{\fml{S}}=\mbf{v}_{\fml{S}}]$ for any $\fml{S}$,
which means 
$\Delta_{i}(\fml{S};\fml{E},\cfn{e}) = -\Delta_{i}(\fml{S};\fml{E}',\cfn{e})$.
Let $\svn{E}(i;\fml{E}',\cfn{e})$ be the SHAP score of feature $i$ in $\fml{E}'$, then we have $\svn{E}(i;\fml{E},\cfn{e}) = -\svn{E}(i;\fml{E}',\cfn{e})$,
which means that $|\svn{E}(i;\fml{E},\cfn{e})| = |\svn{E}(i;\fml{E}',\cfn{e})|$.
Hence, any issue (I1 to I6) that occurs in $\fml{E}$ will also occur in $\fml{E}'$.
\end{proof}

\subsection{Calculations and Proofs for Regression Problems} \label{ssec:calcs}

\subsubsection{Expected values of $\rho_2$:}
~\\

\begin{enumerate}
\item $\fml{S}=\emptyset$:
  \begin{align}
    \exv[\rho_2(\mbf{x})&\,|\,\mbf{x}_{\fml{S}}=\mbf{v}_{\fml{S}}]=
    \nonumber\\
    &
    = \sfrac{1}{4}\int_{-\sfrac{1}{2}}^{\sfrac{3}{2}}\int_{-\sfrac{1}{2}}^{\sfrac{3}{2}}\rho_2(x_1,x_2)d{x_1}d{x_2}
    \nonumber\\
    &
    = \sfrac{1}{4}\left[
      \int_{-\sfrac{1}{2}}^{\sfrac{3}{2}}\int_{\sfrac{1}{2}}^{\sfrac{3}{2}}{x_1}d{x_1}d{x_2}
      +
      \int_{-\sfrac{1}{2}}^{\sfrac{1}{2}}\int_{-\sfrac{1}{2}}^{\sfrac{1}{2}}(x_2-2)d{x_1}d{x_2}
      +
      \int_{\sfrac{1}{2}}^{\sfrac{3}{2}}\int_{-\sfrac{1}{2}}^{\sfrac{1}{2}}(x_2+1)d{x_1}d{x_2}\right]
    \nonumber\\
    &
    = \sfrac{1}{4}\left[
      2\int_{\sfrac{1}{2}}^{\sfrac{3}{2}}{x_1}d{x_1}
      +
      \int_{-\sfrac{1}{2}}^{\sfrac{1}{2}}(x_2-2)d{x_2}
      +
      \int_{\sfrac{1}{2}}^{\sfrac{3}{2}}(x_2+1)d{x_2}\right]
    \nonumber\\
    &
    = \sfrac{1}{4}\left[
      2\bigg\lbrack\sfrac{x_1^2}{2}\bigg\rbrack_{\sfrac{1}{2}}^{\sfrac{3}{2}}
      +
      \bigg\lbrack(\sfrac{x_2^2}{2}-2x_2)\bigg\rbrack_{-\sfrac{1}{2}}^{\sfrac{1}{2}}
      +
      \bigg\lbrack(\sfrac{x_2^2}{2}+x_2)\bigg\rbrack_{\sfrac{1}{2}}^{\sfrac{3}{2}}
      \right]
    \nonumber\\
    &
    = \sfrac{1}{4}\left[
      2(\sfrac{9}{8}-\sfrac{1}{8})
      +
      (\sfrac{1}{8}-\sfrac{1}{8}-2\sfrac{1}{2}-2\sfrac{1}{2})
      +
      (\sfrac{9}{8}-\sfrac{1}{8}+\sfrac{12}{8}-\sfrac{4}{8})
      \right]
    \nonumber\\
    &
    = \sfrac{1}{2}
    \nonumber
  \end{align}
\item $\fml{S}=\{1\}$:
  \begin{align}
    \exv[\rho_2(\mbf{x})\,|\,\mbf{x}_{\fml{S}}=\mbf{v}_{\fml{S}}]
    = \sfrac{1}{2}\int_{-\sfrac{1}{2}}^{\sfrac{3}{2}}\rho_2(1,x_2)d{x_2}
    = \sfrac{1}{2}\left[
      \int_{-\sfrac{1}{2}}^{\sfrac{3}{2}}1d{x_2}
      \right]
    = 1
    \nonumber
  \end{align}
\item $\fml{S}=\{2\}$:

  \begin{align}
    \exv[\rho_2(\mbf{x})&\,|\,\mbf{x}_{\fml{S}}=\mbf{v}_{\fml{S}}]=
    \nonumber\\
    &
    = \sfrac{1}{2}\int_{-\sfrac{1}{2}}^{\sfrac{3}{2}}\rho_2(x_1,x_2)d{x_1}
    \nonumber\\
    &
    = \sfrac{1}{2}\left[
      \int_{\sfrac{1}{2}}^{\sfrac{3}{2}}{x_1}d{x_1}
      +
      \int_{-\sfrac{1}{2}}^{\sfrac{1}{2}}2d{x_1}
      \right]
    \nonumber\\
    &
    = \sfrac{1}{2}\left[
      \bigg\lbrack\sfrac{x_1^2}{2}\bigg\rbrack_{\sfrac{1}{2}}^{\sfrac{3}{2}}
      +
      2\bigg\lbrack{x_1}\bigg\rbrack_{-\sfrac{1}{2}}^{\sfrac{1}{2}}
      \right]
    \nonumber\\
    &
    = \sfrac{1}{2}\left[
      \sfrac{9}{8}-\sfrac{1}{8}
      +
      2
      \right]
    \nonumber\\
    &
    = \sfrac{3}{2}
    \nonumber
  \end{align}

\item $\fml{S}=\{1,2\}$:
  \begin{align}
    \exv[\rho_2(\mbf{x})\,|\,\mbf{x}_{\fml{S}}=\mbf{v}_{\fml{S}}]
    = \rho_2(1,1) = 1
    \nonumber
  \end{align}
\end{enumerate}

\subsubsection{Proof regarding $\rho_3$:}

\begin{proposition}
  $\rho_3$ (see~\cref{fig:rmlc}) is Lipschitz-continuous.
\end{proposition}

\begin{proof}(Sketch)
  $\rho_3$ is composed of continuously glued 1-degree 1 polynomials.
  It is well known that 1-degree 1 polynomials are
  Lipschitz-continuous (the exact constant depends on the distance
  used on $\mbb{R}^2$) and that gluing several Lipschitz-continuous
  functions is still a Lipschitz-continuous function (and the constant is
  the maximum of the constants of the glued functions).\\
  (Obs: computing the actual constants does not provide relevant
  insights, and is considerably laborious.
\end{proof}

\end{document}